%% file: aaai25.tex
\documentclass[letterpaper]{article} 
\usepackage{aaai25}  
\usepackage{times}  
\usepackage{helvet}  
\usepackage{courier}  
\usepackage[hyphens]{url}  
\usepackage{graphicx} 
\usepackage{natbib}  
\usepackage{caption} 
\DeclareCaptionStyle{ruled}{labelfont=normalfont,labelsep=colon,strut=off}

\usepackage{amsmath}
\usepackage{amssymb}
\usepackage{xspace}
\usepackage[ruled,vlined,linesnumbered]{algorithm2e}
\usepackage{multirow}
\usepackage{booktabs}
\usepackage{mathtools}
\usepackage{threeparttable}
\usepackage{mdwlist}
\usepackage{subcaption}

\usepackage{xcolor}
\usepackage{amsthm}
\usepackage{enumitem}
\usepackage{footmisc}

\newtheorem{problem}{Problem}

\newtheorem{theorem}{Theorem}[]
\newtheorem{lemma}{Lemma}[]

\theoremstyle{definition}

\SetCommentSty{mycommfont}

\DeclareMathOperator*{\argmin}{arg\,min}

\newcommand{\mc}[1]{\mathcal{#1}\xspace}
\newcommand{\mb}[1]{\mathbf{#1}\xspace}

\newcommand{\method}[0]{\textsc{PULL}\xspace}
\newcommand{\methodlong}[0]{\textsc{\underline{PU}-\underline{L}earning-based \underline{L}ink predictor}\xspace}

\newcommand{\ul}[1]{\underline{#1}}

\newcommand{\blue}[1]{{\color{black} #1}}

\newcommand{\appendixtitlepage}{
    \noindent
    \begin{center}
        {\LARGE \bfseries Appendix}\\[1cm]
    \end{center}
}

\title{Accurate Link Prediction for Edge-Incomplete Graphs via PU Learning}
\author{
	Junghun Kim\textsuperscript{\rm 1},
	Ka Hyun Park\textsuperscript{\rm 1},
	Hoyoung Yoon\textsuperscript{\rm 1},
	U Kang\textsuperscript{\rm 1}
}
\affiliations{
    \textsuperscript{\rm 1}Seoul National University, South Korea\\
    \{bandalg97, kahyunpark, crazy8597, ukang\}@snu.ac.kr, 
}

\begin{document}

\maketitle

\begin{abstract}
Given an edge-incomplete graph, how can we accurately find its missing links?
The problem aims to discover the missing relations between entities when their relationships are represented as a graph.
Edge-incomplete graphs are prevalent in real-world due to practical limitations, such as not checking all users when adding friends in a social network.
Addressing the problem is crucial for various tasks, including recommending friends in social networks and finding references in citation networks.
However, previous approaches rely heavily on the given edge-incomplete (observed) graph, making it challenging to consider the missing (unobserved) links.

In this paper, we propose \method, an accurate link prediction method based on the positive-unlabeled (PU) learning.
\method treats the observed edges in the training graph as positive examples, and the unconnected node pairs as unlabeled ones.
\method prevents the model from blindly trusting the observed graph by proposing latent variables for unconnected node pairs, and leveraging the expected graph structure with respect to these variables.
Extensive experiments on real-world datasets show that \method consistently outperforms the baselines for predicting links in edge-incomplete graphs.
\end{abstract}

\input{010intro}

\input{020related}

\input{030method}

\input{040experiment}

\input{050conclusion}

\blue{
\section*{Acknowledgements}
This work was supported by Institute of Information \& communications Technology Planning \& Evaluation (IITP) grant funded by the Korea government (MSIT) [No.2022-0-00641, XVoice: Multi-Modal Voice Meta Learning], [No.RS-2020-II200894, Flexible and Efficient Model Compression Method for Various Applications and Environments], [No.RS-2021-II211343, Artificial Intelligence Graduate School Program (Seoul National University)], [No.RS-2024-00509257, Global AI Frontier Lab], and [No.RS-2021-II212068, Artificial Intelligence Innovation Hub (Artificial Intelligence Institute, Seoul National University)]. 
The Institute of Engineering Research at Seoul National University and the ICT at Seoul National University provided research facilities for this work.
The code and datasets are available at \url{https://github.com/snudatalab/PULL}.
U Kang is the corresponding author.
}

\bibliography{aaai25}

\clearpage
\pagebreak
\onecolumn
\appendixtitlepage
\input{060supp}

\end{document}

%% file: 010intro.tex
\section{Introduction}
\label{sec:introduction}

\emph{Given an edge-incomplete graph, how can we accurately find the missing links among the unconnected node pairs?}
Edge-incomplete graphs are easily encountered in real-world networks.
In social networks, connections between users can be missing since we do not check every user when adding friends.
In the context of citation networks, there may be missing citations as we do not review all published papers for citation.
In these scenarios, the objective is to find uncited references in citation networks~\cite{ShibataKS12, LiuKYQ19} or to recommend friends in social networks~\cite{DBLP:journals/chinaf/WangXWZ15, DaudHSSA20}.

The main limitation of previous works~\cite{ZhangLXWJ21, ZhuZXT21, ChamberlainSRFM23, abs-2303-08958} for link prediction is that they assume the unconnected edges in the given graph as true negative examples. 
In link prediction, the complete set of true unconnected edges are not given.
The given graph contains only a subset of ground-truth edges, while the other node pairs remain unlabeled; i.e., true unconnected edges and the unobserved ground-truth edges are mixed without labels. 
This misconception limits the model's ability to propagate information through unconnected node pairs, which may potentially form edges, resulting in over-reliance of a link predictor to the edge-incomplete graph.
Thus, it is important to consider the uncertainties of the given graph to obtain accurate linking probabilities.

\begin{figure}[t]
	\centering
	\includegraphics[width=0.475\textwidth]{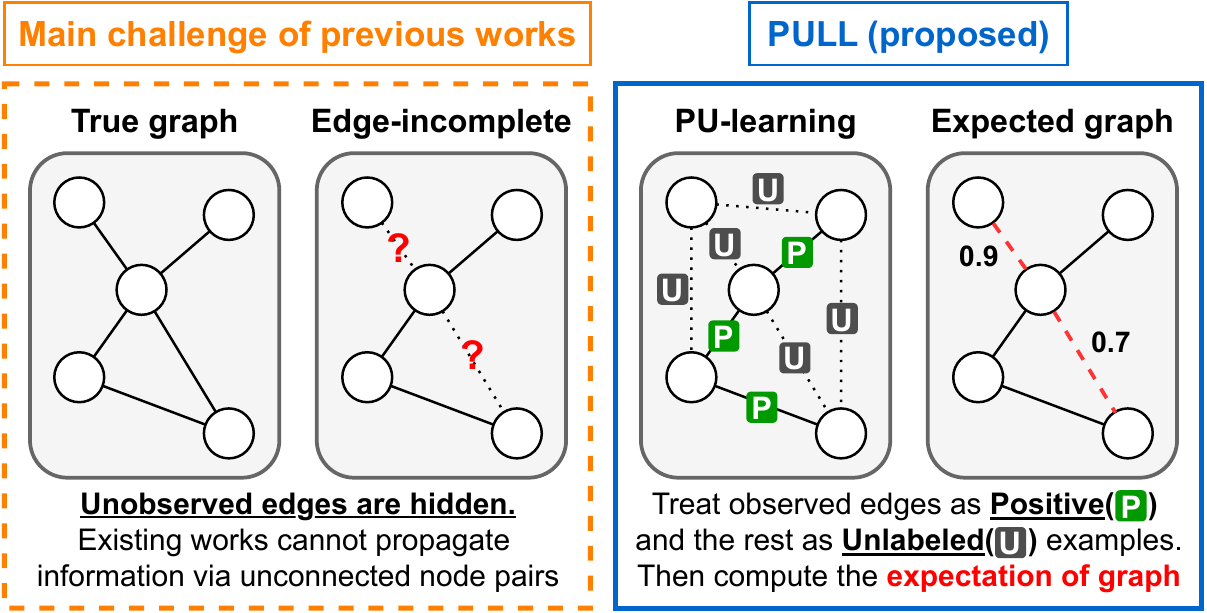}
	\caption{Main challenge of previous works.
		They cannot consider the hidden unobserved edges in the given graph.
		\method treats the unconnected node pairs as unlabeled examples, and utilizes the expectation of graph structure. 		
	}
	\label{figure:crown}
\end{figure}

In this work, we propose \method (\methodlong), an accurate link prediction method in edge-incomplete graphs.
To account for the uncertainties in the given graph while training a link predictor, \method exploits PU (Positive-Unlabeled) learning (see Related Works for details).
We treat the observed edges in the given graph as positive examples and the unconnected node pairs, which may contain hidden edges, as unlabeled examples.
We then construct an expected graph while proposing latent variables for the unlabeled (unconnected) node pairs to consider the hidden edges among them.
This enables us to effectively propagate information through the unconnected edges, improving the prediction accuracy.
The main challenge of previous works and our approach is depicted in Figure~\ref{figure:crown}. 


\begin{figure*}[t]
	\centering
	\includegraphics[width=0.85\textwidth]{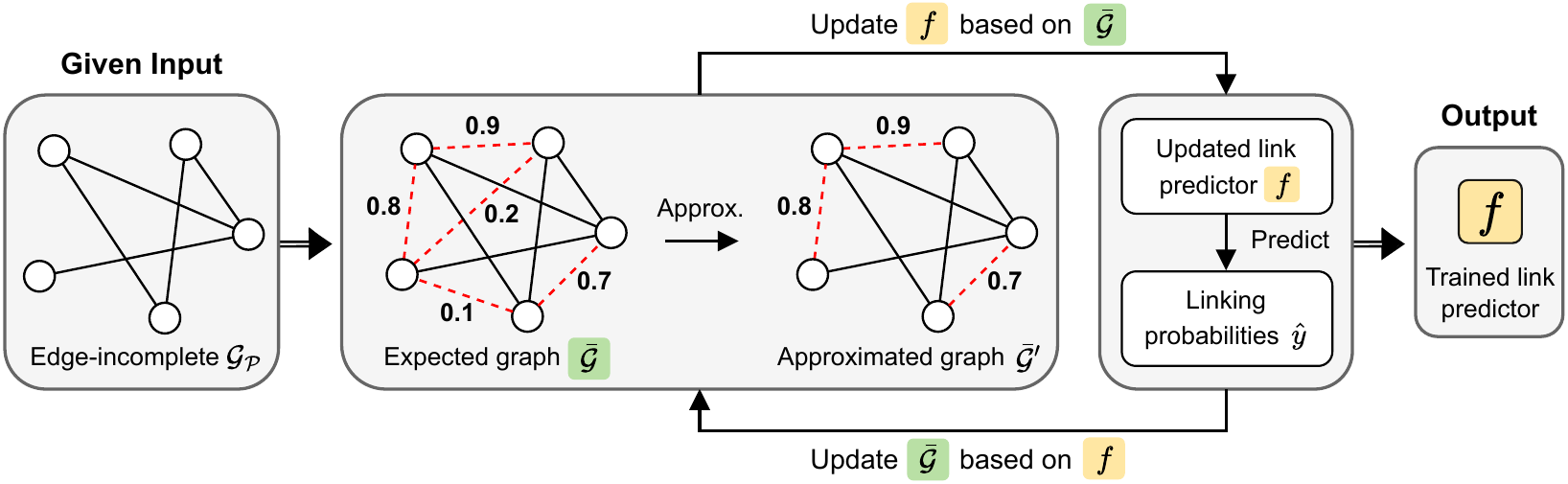}
	\caption{
		Overall structure of \method.
		Given an edge-incomplete graph $\mc{G}_\mc{P}$ with a set $\mc{P}$ of observed edges, \method first computes the expected graph structure $\mc{\bar{G}}$ by proposing latent variables for the edges.
		Then \method utilizes $\mc{\bar{G}}$ to update the link predictor $f$.
		The marginal linking probabilities $\hat{y}$ obtained by the updated $f$ are used to compute $\mc{\bar{G}}$ in the next iteration.
	}
	\label{figure:overall}
\end{figure*}

Our contributions are summarized as follows:
\begin{itemize}
\item \textbf{Method.} We propose \method, an accurate link prediction method in graphs.
	\method effectively overcomes the primary limitation of previous methods, specifically their heavy reliance on the provided edge-incomplete graph.
  \item \textbf{Theory.} We theoretically analyze \method, including the relationship with the EM algorithm and the complexity.
  \item \textbf{Experiments.} Extensive experiments on real-world datasets show that \method achieves the best performance.
\end{itemize}



%% file: 020related.tex
\section{Related Works}
\label{sec:related}

\subsection{Link Prediction in Graphs}
\label{ssec:related-link-prediction}

\blue{Many graph convolutional networks have been proposed to perform various tasks in graph-structured data including graph classification~\cite{YooSK22}, node feature estimation~\cite{YooJJK22}, graph generation~\cite{JungPK20}, and model compression~\cite{kim2021compressing}.
Recently, link prediction has garnered significant attention due to its successful application in various domains including social networks~\citep{DaudHSSA20}, recommendation systems~\citep{afoudi2023enhanced}, and biological networks~\citep{LongWLFKCLL22}.}
Previous link prediction approaches are categorized into two groups: autoencoder-based and embedding-based approaches.

Autoencoder methods use autoencoder structure~\citep{hinton2006reducing} for the link predictor.
GAE and VGAE~\shortcite{KipfW17} are the first autoencoder-based unsupervised framework for link prediction.
ARGA and ARGVA~\shortcite{PanHLJYZ18} exploit adversarial training strategy to improve the performance of GAE and VGAE, respectively.
VGNAE~\shortcite{AhnK21} finds that autoencoder-based methods produce embeddings that converge to zero for isolated nodes.
They utilize L2-normalization to get better embeddings for these isolated nodes.
MVGAE~\shortcite{LiLWL23} is a multi-view representation model which considers both global and local topologies.
However, those autoencoder-based methods have limitations in that they cannot consider the missing edges during training.

Embedding-based approaches create compact representations of edges through information aggregation or subgraph extraction.
GCN~\shortcite{KipfW17}, GraphSAGE~\shortcite{HamiltonYL17}, and GAT~\shortcite{Petar17} aggregate information from neighbors to learn the edge embeddings, assuming adjacent nodes are similar.
1D-GNN~\shortcite{you2021identity} formalizes link prediction as a conditional node classification, incorporating the identity of the source node. 
Neo-GNN~\shortcite{abs-2206-04216} propagates information through the original graph and concatenates structural features such as the count of common neighbors.
SEAL~\shortcite{ZhangC18} and BUDDY~\shortcite{ChamberlainSRFM23} extend the link prediction problem to a subgraph classification problem.
PS2~\shortcite{TanZLZL0CH23} automatically identifies optimal subgraphs for different edges.
NESS~\shortcite{abs-2303-08958} utilizes static subgraphs during training and aggregates their representations at test time to learn embeddings in a transductive setting.
However, the primary limitation of those methods is that they assume the edges of the given graph are fully observed.
This misconception restricts the model's ability to propagate information between unconnected node pairs that might form edges, leading to over-reliance of the link predictor to the edge-incomplete graph.
%

\subsection{Graph-based PU Learning}
\label{ssec:related-pu-learning}


The objective of PU (Positive-Unlabeled) learning is to train a binary classifier that effectively distinguishes positive and negative instances when only positive and unlabeled examples are available~\cite{KiryoNPS17, ZhaoXJWH22}.
Recently, many graph-based PU learning approaches have been studied~\citep{MaZ17, ZhangRL0019, WuPDTZD19}.
GRAB~\shortcite{YooKYKJK21, YooKYKJK22} is the first approach to solve the graph-based PU learning problem without knowing the class prior in advance.
PULNS~\shortcite{LuoZCQDZWCHRL21} uses reinforcement learning to design an effective negative sample selector.
PU-GNN~\shortcite{YangZYK23} introduces distance-aware PU loss to achieve more accurate supervision.
However, those methods cannot be directly used in the link prediction problem since they aim to classify nodes, not edges, while considering the edges of the given graph as fully observed ones.
PU-AUC~\shortcite{hao2021learning} and Bagging-PU~\shortcite{gan2022positive} proposed PU learning frameworks for link prediction considering the given edges as observed positive examples.
However, their performance is constrained by the propagation of information through the edge-incomplete graph for obtaining node and edge representations.

%% file: 030method.tex
\section{Proposed Method}
\label{sec:proposed-method}

We propose \method, an accurate method for link prediction.
We illustrate the overall process of \method in Figure~\ref{figure:overall} and Algorithm~\ref{algorithm:proposed-method}.
The main challenges and our approaches are:
\begin{enumerate}[]
	\item[\textbf{C1.}] \textbf{How can we consider the missing links?}
		We treat the given edges as observed positive examples, and the rest (unconnected edges) as unlabeled ones.
		We then propagate information through \emph{expectation of graph structure} by proposing latent variables to the unconnected edges.
	\item[\textbf{C2.}] \textbf{How can we effectively model the expected graph?}
	Naive computation of the expected graph is intractable since there are $2^{\mid\mc{E}_\mc{U}\mid}$ possible structures where $\mc{E}_\mc{U}$ is the set of unconnected edges.
		We compute the expectation of graph by carefully designing the probabilities of graphs.
	\item[\textbf{C3.}] \textbf{How can we gradually improve the performance of the link predictor?}
		\method iteratively improves the quality of the expected graph structure, which is the evidence for training the link predictor.
\end{enumerate}

\subsection{Modeling Missing Links (C1)}
\label{ssec:method-pu-learning}

In a link prediction problem, we are given a feature matrix $\mb{X}$ and an edge-incomplete graph $\mc{G}_\mc{P}$ consisting of two sets of edges, $\mc{E}_\mc{P}$ and $\mc{E}_\mc{U}$.
The set $\mc{E}_\mc{P}$ contains observed edges, while $\mc{E}_\mc{U}$ consists of unconnected node pairs; $\mc{E}_\mc{P} \cup \mc{E}_\mc{U}$ is a set of all possible node pairs.
Then we aim to find unobserved connected edges among $\mc{E}_\mc{U}$ accurately.
Existing link prediction methods treat $\mc{E}_\mc{U}$ as true negative examples, which restricts the model's ability to propagate information through unconnected node pairs that could potentially form edges.
This misleads $f$ into fitting the edge-incomplete graph, thereby degrading the prediction performance.

\method addresses this misconception of previous methods by modeling the given graph based on PU-learning.
Since there are hidden connections in $\mc{E}_\mc{U}$, we treat the unconnected edges in $\mc{E}_\mc{U}$ as unlabeled examples, and the observed edges in $\mc{E}_\mc{P}$ as positive ones.
Then we propose a latent variable $z_{ij} \in \{ 1, 0 \}$ for every edge $e_{ij}$, indicating whether there is a link between nodes $i$ and $j$ to consider the hidden connections; $z_{ij}=1$ for every $e_{ij} \in \mc{E}_\mc{P}$, but not always $z_{ij}=0$ for $e_{ij} \in \mc{E}_\mc{U}$.
We denote the graph $\mc{G}_\mc{P}$ with latent variable $\mb{z} = \{z_{ij} \mathrm{\;for\;} e_{ij} \in (\mc{E}_\mc{P} \cup \mc{E}_\mc{U})\}$ as $\mc{G}_\mc{P}(\mb{z})$.

A main challenge is that we cannot propagate information through the \emph{variablized graph} $\mc{G}_\mc{P}(\mb{z})$ since every edge $e_{ij} \in \mc{E}_\mc{U}$ of $\mc{G}_\mc{P}(\mb{z})$ is probabilistically connected.
Instead, \method exploits the expectation $\bar{\mc{G}}$ of graph $\mc{G}_\mc{P}(\mb{z})$ over the latent variables $\mb{z}$.
This enables us to train a link predictor $f$ accurately, considering the hidden connections in $\mc{E}_\mc{U}$.
Given that the estimated linking probabilities of $f$ provide prior knowledge for constructing the expected graph $\mc{\bar{G}}$, improved link predictor $f$ enhances the quality of $\mc{\bar{G}}$.
Thus, \method employs an iterative learning approach with two-steps to achieve a repeated improvement of the link predictor: a) constructing an expected graph $\mc{\bar{G}}$ based on the predicted linking probabilities from $f$, and b) updating $f$ utilizing $\mc{\bar{G}}$.
The updated $f$ is used to refine $\mc{\bar{G}}$ in the subsequent iteration.

\begin{algorithm}[t]
\SetKwInOut{Input}{Input}
\SetKwInOut{Output}{Output}
\SetAlgoLined
\small
	\Input{Edge-incomplete graph $\mc{G}_\mc{P} = (\mc{V}, \mc{E}_\mc{P})$, feature matrix $\mb{X}$, set $\mc{E}_\mc{U}$ of unconnected edges, hyperparameter $r$, and link predictor $f_\theta$}
	\Output{Best parameters $\theta$ of the link predictor $f_\theta$}
	\BlankLine
	Randomly initialize $\theta^\mathrm{new}$ , and initialize $K$ as $|\mc{E}_\mc{P}|$\;
	\Repeat{the maximum number of iterations is reached or the early stopping condition is met}{
		$\theta \leftarrow \theta^\mathrm{new}$\;
			$\bar{\mc{G}} \leftarrow \mathbb{E}_{\mb{z} \sim p(\mb{z} \mid \mb{X}, \mc{E}_\mc{P}, \theta)} [\mb{A}(\mb{z})] = \mb{A}^\mc{\bar{G}}$
		\tcp*{Equations~(\ref{eq:expected-graph-structure}, \ref{eq:expected-graph-structure-2})}
		Approximate $\bar{\mc{G}}$ to $\bar{\mc{G}}'$ by selecting $K$ confident edges from a set of candidate edges\;
		$K \leftarrow K + |\mc{E}_\mc{P}|*r$ \;
		$\theta^\mathrm{new} \leftarrow \argmin_\theta \mc{L}(\theta; \bar{\mc{G}}', \mb{X})$\tcp*{Equations~(\ref{eq:objective-0}, \ref{eq:objective-2})}
	}
 	\caption{Overall process of \method.}
	\label{algorithm:proposed-method}
\end{algorithm}

\subsection{Expectation of Graph Structure (C2)}
\label{ssec:method-expected-structure}

During training of a link predictor $f_\theta$, \method propagates information through the expected graph $\bar{\mc{G}}$ of $\mc{G}_\mc{P}(\mb{z})$ over the latent variable $p(\mb{z} \mid \mb{X}, \mc{E}_\mc{P}, \theta)$ where $\theta$ is the learnable model parameter.
Note that $\bar{\mc{G}}$ requires computing the joint probabilities $p(\mb{z} \mid \mb{X}, \mc{E}_\mc{P}, \theta)$ for all possible graph structures $\mc{G}_\mc{P}(\mb{z})$.
This is intractable since there are $2^{\mid\mc{E}_\mc{U}\mid}$ possible states of $\mb{z}$ in $\mc{G}_\mc{P}(\mb{z})$.
Instead, \method efficiently computes the expectation of graph structure by carefully designing the joint probability $p(\mb{z} \mid \mb{X}, \mc{E}_\mc{P}, \theta)$.
In the following, we describe how we design the joint probability and compute the expectation of graph.

\subsubsection{Designing the joint probability.}
We convert the graph $\mc{G}_\mc{P}(\mb{z})$ into a line graph $L(\mc{G}_\mc{P}(\mb{z})) = (\mc{V}_L, \mc{E}_L)$ where nodes in $L(\mc{G}_\mc{P}(\mb{z}))$ represent the edges of $\mc{G}_\mc{P}(\mb{z})$, and two nodes in $L(\mc{G}_\mc{P}(\mb{z}))$ are connected if their corresponding edges in $\mc{G}_\mc{P}(\mb{z})$ are adjacent.
$\mc{V}_L$ contains both $\mc{E}_\mc{P}$ and $\mc{E}_\mc{U}$ of $\mc{G}_\mc{P}(\mb{z})$ since every node pair $(i, j)$ in $\mc{G}_\mc{P}(\mb{z})$ is correlated with variable $z_{ij}$.
We then consider the line graph as a pairwise Markov network, which assumes that any two random variables in the network are conditionally independent of each other given the rest of the variables if they are not directly connected~\citep{koller2009probabilistic}.
This simplifies the probabilistic modeling on graph-structured random variables, and effectively marginalizes the joint distribution of nodes in $L(\mc{G}_\mc{P}(\mb{z}))$, which corresponds to the distribution $p(\mb{z} \mid \mb{X}, \mc{E}_\mc{P}, \theta)$ of edges in the original graph $\mc{G}_\mc{P}(\mb{z})$.

With the Markov property, the distribution $p(\mb{z} \mid \mb{X}, \mc{E}_\mc{P}, \theta)$ is computed by the multiplication of all the node and edge potentials in the line graph $L(\mc{G}_\mc{P}(\mb{z}))$:
\begin{small}
\begin{multline}
	p(\mb{z} \mid \mb{X}, \mc{E}_\mc{P}, \theta) = \\
	 \frac{1}{F} \prod_{ij \in \mc{V}_L} \phi_{ij}(z_{ij} \mid \mb{X}, \theta) \prod_{(ij,jk) \in \mc{E}_L} \psi_{ij, jk}(z_{ij}, z_{jk} \mid \mb{X}, \theta)
\label{equation:mrf-1}
\end{multline}
\end{small}
where $\phi_{ij}$ is the node potential for a transformed node $ij$, and $\psi_{ij, jk}$ is the edge potential for a transformed edge $(ij,jk)$.
The node potential $\phi_{ij}$ represents the unnormalized marginal linking probability between nodes $i$ and $j$ in the original graph $\mc{G}_\mc{P}(\mb{z})$.
The edge potential $\psi_{ij, jk}$ denotes a degree of homophily between the edges containing a common node in $\mc{G}_\mc{P}(\mb{z})$.
$F$ is the normalizing factor that ensures the distribution adds up to one.
For simplicity, we omit $\mb{X}$ in $\phi_{ij}$ and $\psi_{ij, jk}$ in the rest of the paper.

We define the node potential $\phi_{ij}$ of $L(\mc{G}_\mc{P}(\mb{z}))$ as follows to make nodes in $\mc{G}_\mc{P}(\mb{z})$ with similar hidden representations have a higher likelihood of connection:
\begin{align*}
\small
	\phi_{ij} (z_{ij}=1 \mid \theta) =
\begin{dcases*}
	1 & if $e_{ij} \in \mc{E}_\mc{P}$ \\
	f_\theta(i, j) & otherwise \\
\end{dcases*}
\label{eq:node-potential}
\end{align*}
where $\phi_{ij} (z_{ij} = 0 \mid \theta) = 1 - \phi_{ij} (z_{ij} = 1 \mid \theta)$ and $f_\theta(i, j)$ is the predicted marginal linking probability between nodes $i$ and $j$.
We set $\phi_{ij} (z_{ij} = 1 \mid \theta) = 1$ for $e_{ij} \in \mc{E}_\mc{P}$ since the linking probability of an observed edge is 1.
We use a GCN followed by a sigmoid function as the link predictor: $f_\theta(i, j) = \sigma(h_i \cdot h_j)$ where $h_i$ is the hidden representation of node $i$ computed by the GCN embedding function with graph $\bar{\mc{G}}$.
In the first iteration, we initialize $\bar{\mc{G}}$ with $\mc{G}_\mc{P}$.
Other graph-based models can also be used instead of GCN (see Q2 of Experiments section).
We define $\psi_{ij, jk}$ as a constant $c$ to make the joint distribution focus on the marginal linking probabilities.
Then the normalizing constant $F$ in Equation~\eqref{equation:mrf-1} becomes $c^{|\mc{E}_L|}$ as $\sum_{\mb{z}} \prod_{ij \in \mc{V}_L} \phi_{ij}(z_{ij} \mid \theta)=1$ (see Lemma~1 in Appendix for proof).

As a result, the probability $p(\mb{z} \mid \mb{X}, \mc{E}_\mc{P}, \theta)$ is expressed by the multiplication of node potentials $\phi_{ij}$ for $e_{ij} \in \mc{E}_\mc{U}$:
\begin{equation}
\small
\begin{split}
	p(\mb{z} \mid \mb{X}, \;& \mc{E}_\mc{P}, \theta) = \prod_{ij \in \mc{V}_L} \phi_{ij}(z_{ij} \mid \theta) \\
	&= \prod_{e_{ij} \in (\mc{E}_\mc{P} \cup \mc{E}_\mc{U})}\phi_{ij}(z_{ij} \mid \theta) = \prod_{e_{ij} \in \mc{E}_\mc{U}}\phi_{ij}(z_{ij} \mid \theta).	
\end{split}
\label{equation:mrf-2}
\end{equation}

\subsubsection{Computing the expectation of graph.}
Let $\mb{A}(\mb{z})$ be the adjacency matrix representing the state $\mb{z}$ where the $(i,j)$-th component of $\mb{A}(\mb{z})$, which we denote as $\mb{A}(\mb{z})_{ij}$, is $z_{ij} \in \{1, 0\}$.
Then the corresponding weighted adjacency matrix $\mb{A}^\mc{\bar{G}}$ of the expected graph $\mc{\bar{G}}$ is computed as follows:
\begin{equation}
\small
\begin{split}
	\mb{A}^\mc{\bar{G}} = \mathbb{E}_{\mb{z} \sim p(\mb{z} \mid \mc{E}_\mc{P}, \theta)} [\mb{A}(\mb{z})] &= \sum_{\mb{z}} p( \mb{z} \mid \mb{X}, \mc{E}_\mc{P}, \theta) \mb{A}(\mb{z}) \\
	&= \sum_{\mb{z}} \prod_{e_{ij} \in \mc{E}_\mc{U}} \phi_{ij}(z_{ij} \mid \theta) \mb{A}(\mb{z}).
\end{split}
\label{eq:expected-graph-structure}
\end{equation}
The $(i, j)$-th component $\mb{A}^\mc{\bar{G}}_{ij}$ of $\mb{A}^\mc{\bar{G}}$ is expressed as
\begin{equation}
\small
\begin{split}
	&\phi_{ij}(z_{ij}=1 \mid \theta) \sum_{\scriptscriptstyle \mb{z} \mid z_{ij}=1} \prod_{\scriptscriptstyle e_{kl} \in \mc{E}_\mc{U} \setminus \{e_{ij}\}} \phi_{kl}(z_{kl} \mid \theta) \mb{A}(\mb{z})_{ij} \\
	&= \phi_{ij}(z_{ij}=1 \mid \theta)
\end{split}
\label{eq:expected-graph-structure-2}
\end{equation}
where the equality holds since $\mb{A}(\mb{z})_{ij}=1$ for $z_{ij}=1$, and $\sum_{\mb{z} \mid z_{ij}=1} \prod_{e_{kl} \in \mc{E}_\mc{U} \setminus \{e_{ij}\}} \phi_{kl}(z_{kl} \mid \theta) = 1$ (see Lemma~1 in Appendix for proof).
As a result, we simply express the expected graph $\mc{\bar{G}}$ by an weighted adjacency matrix $\mb{A}^\mc{\bar{G}}$ where $\mb{A}^\mc{\bar{G}}_{ij} = \phi_{ij}(z_{ij}=1 \mid \theta)$.

\subsubsection{Approximating the expected graph.}
Propagating information through $\mc{\bar{G}}$ directly may lead to oversmoothing problem, as $\mc{\bar{G}}$ is a fully connected graph represented by $\mb{A}^\mc{\bar{G}}$.
Moreover, the training time increases exponentially with the number of nodes.
To address these challenges, \method utilizes an approximated one of $\mc{\bar{G}}$, which contains edges with high confidence.
Specifically, we keep the top-$K$ edges with the largest weights while removing the rest.
We denote this approximated one as $\bar{\mc{G}}'$, and its adjacency matrix as $\mb{A}^{\mc{\bar{G}}'}$.
From the perspective of PU learning, selecting edges in $\mc{\bar{G}}$ can be viewed as selecting relatively connected edges among the unlabeled ones, while treating the rest as relatively unconnected edges.
We gradually increase the number $K$ of selected edges in proportion to that of observed edges through the outer iteration of Algorithm~\ref{algorithm:proposed-method}, which is expressed by $K \leftarrow K + r|\mc{E}_\mc{P}|$, giving more trust in the expected graph $\bar{\mathcal{G}}$.
This is because the quality of $\bar{\mathcal{G}}$ improves through the iterations (Figure~\ref{figure:iteration}).
We set $r=0.05$ in our experiments.
%

Another challenge lies in the need to compute weights for every node pair in each outer iteration to acquire the top-$K$ edges, which results in computational inefficiency.
To address this, we define a set of candidate edges determined by the node degrees.
This stems from the observation that nodes with higher degrees exhibit a greater likelihood of forming new connections in real-world networks~\cite{Albert99}.
The candidate edge set consists of node pairs where at least one node has top-$M$ degree among all the nodes.
We set $M=100$ in our experiments.
\method selects top-$K$ edges among the candidate edge set instead of all node pairs to approximate $\mc{\bar{G}}$.


\subsection{Iterative Learning of Link Predictor (C3)}
\label{ssec:method-loss}

At each iteration, \method constructs the expected graph $\bar{\mc{G}}$ given a trained link predictor $f_{\theta}$ with current parameter $\theta$.
Then \method propagates information through $\bar{\mc{G}}'$, which is the approximated version of $\bar{\mc{G}}$, to train a new link predictor $f_{\theta^\mathrm{new}}$ with new parameter $\theta^\mathrm{new}$.
This effectively prevents the link predictor from blindly trusting the given edge-incomplete $\mc{G}_\mc{P}$.

To optimize the new parameter $\theta^\mathrm{new}$, we propose the binary cross entropy (BCE) loss $\mc{L}_E$ in Equation~\eqref{eq:objective-0} by treating the given edges in $\mc{E}_\mc{P}$ and the unconnected edges in $\mc{E}_\mc{U}$ as positive and unlabeled (PU) examples, respectively.
For the unconnected edges, we use the expected linking probabilities $\mb{A}^\mc{\bar{G}'}_{ij}$ as pseudo labels for $e_{ij}$:
\begin{equation}
\small
\begin{split}
	\mc{L}_E &= - \sum_{e_{ij} \in \mc{E}_\mc{P}} \log \hat{y}_{ij} -\sum_{e_{ij} \in \mc{E}_{\mc{U}}^r} \log (1-\hat{y}_{ij}) \\
		& \quad - \sum_{e_{ij} \in \mc{E}_\mc{P}^r} \bigl( \mb{A}^\mc{\bar{G}'}_{ij} \log \hat{y}_{ij} + (1-\mb{A}^\mc{\bar{G}'}_{ij}) \log (1-\hat{y}_{ij}) \bigr)
\end{split}
\label{eq:objective-0}
\end{equation}
where $\hat{y}_{ij} = f_{\theta^\mathrm{new}}(i, j)$.
$\mc{E}_\mc{P}^r$ is the set of relatively connected edges selected from $\mc{E}_\mc{U}$ when approximating the expected graph structure $\mc{\bar{G}}$ by $\mc{\bar{G}}'$, and $\mc{E}_\mc{U}^r = \mc{E}_\mc{U} \setminus \mc{E}_\mc{P}^r$.

However, in real-world graphs, there is a severe imbalance between the numbers of connected edges and unconnected ones.
We balance them by randomly sampling $|\mc{E}_\mc{P} \cup \mc{E}_\mc{P}^r|$ unconnected edges among $\mc{E}_\mc{U}^r$ for every epoch.
We denote the loss $\mc{L}_E$ with a set $\mc{E}_{\mc{U}}'$ of randomly sampled edges among $\mc{E}_{\mc{U}}^r$ with size $|\mc{E}_{\mc{U}}'| = |\mc{E}_\mc{P} \cup \mc{E}_\mc{P}^r|$ as $\mc{L}'_E$.

If the current parameter $\theta$ of the link predictor is inaccurate, the quality of the expected graph structure deteriorates, leading to the next iteration's parameter $\theta^\mathrm{new}$ becoming even more inaccurate.
Thus, we propose another loss term $\mc{L}_C$ for correction, which measures the BCE for all observed edges and randomly sampled unconnected edges from $\mc{E}_\mc{U}^r$:
\begin{equation}
\small
	\mc{L}_C = - \sum_{e_{ij} \in \mc{E}_\mc{P}} \log \tilde{y}_{ij} - \sum_{e_{kl} \in \mc{E}_\mc{U}''} \log (1-\tilde{y}_{ij})
\label{eq:objective-2}
\end{equation}
where $\mc{E}_\mc{U}''$ is the set of randomly sampled node pairs from $\mc{E}_\mc{U}^r$ with size $|\mc{E}_\mc{U}''| = |\mc{E}_\mc{P}|$.
$\tilde{y}_{ij}$ is $f_{\theta^\mathrm{new}}(i, j)$ computed with the given graph $\mc{G}_\mc{P}$.
$\mathcal{L}_C$ effectively prevents excessive self-reinforcement in the link predictor of \method (Figure~\ref{figure:iteration}).

As a result, \method finds the best parameter $\theta^\mathrm{new}$ for each iteration by minimizing the sum of the two loss terms $\mathcal{L}'_E$ and $\mathcal{L}_C$.
We denote the final loss function as $\mc{L}(\theta^\mathrm{new}; \mc{\bar{G}}', \mb{X}) = \mc{L}_E' + \mc{L}_C$.
The new parameter $\theta^\mathrm{new}$ is used as the current $\theta$ for the next iteration.
The iterations stop if the maximum number of iterations is reached or an early stopping condition is met.

\subsection{Theoretical Analysis}
\label{ssec:theory}

We theoretically analyze \method in terms of its connection to the EM algorithm, and the time complexity.

\subsubsection{Relation of \method to EM algorithm.}
EM (Expectation-Maximization) \cite{dempster1977maximum} is an iterative method used for estimating model parameter $\theta$ when there are unobserved data.
It assigns latent variables $\mb{z}$ to the unobserved data, and maximizes the expectation of the log likelihood $\log p(\mb{y}, \mb{z} \mid \mb{X}, \theta)$ in terms of $\mb{z}$ to optimize $\theta$ where $\mb{y}$ and $\mb{X}$ are target and input variables, respectively.

In our problem, the expectation of the log likelihood given the current parameter $\theta$ is written as follows:
\begin{equation}
	Q(\theta^\mathrm{new} \mid \theta) =
	\mathbb{E}_{\mb{z} \sim p(\mb{z} \mid \mb{X}, \mc{E}_\mc{P}, \theta)} [\log p(\mc{E}_\mc{P}, \mb{z} \mid \mb{X}, \theta^\mathrm{new})]
\label{eq:em-1}
\end{equation}
where $\theta^\mathrm{new}$ is the new parameter.
The EM algorithm finds $\theta^\mathrm{new}$ that maximizes $Q(\theta^\mathrm{new} \mid \theta)$, and they are used as $\theta$ in the next iteration.
The algorithm is widely used in situations involving latent variables as it always improves the likelihood $Q$ through the iterations~\citep{murphy2012machine}.

\method iteratively optimizes the parameter $\theta$ of a link predictor by minimizing both $\mc{L}_E'$ and $\mc{L}_C$ where $\mc{L}_E'$ is the approximation of $\mc{L}_E$ in Equation~\eqref{eq:objective-0}.
We compare Equations~(\ref{eq:objective-0}) and (\ref{eq:em-1}) to show the similarity between the iterative minimization of $\mc{L}_E$ in \method and the iterative maximization of $Q(\theta^\mathrm{new} \mid \theta)$ in the EM algorithm.

\method expresses $p(\mb{z} \mid \mb{X}, \mc{E}_\mc{P}, \theta)$ in Equation~\eqref{eq:em-1} by the multiplication of node potentials as in Equation~\eqref{equation:mrf-2}.
%
For the joint probability $p(\mc{E}_\mc{P}, \mb{z} \mid \mb{X}, \theta^\mathrm{new})$ in Equation~\eqref{eq:em-1}, we approximate it using a link predictor $f_{\theta^\mathrm{new}}$ with new parameter $\theta^\mathrm{new}$.
We consider $f_{\theta_\mathrm{new}}$ as a marginalization function that gives marginal linking probabilities for each node pair.
We also assume that the marginal distributions obtained by $f_{\theta_\mathrm{new}}$ are mutually independent.
Then the joint probability is approximated as follows:
\begin{equation}
\small
	p(\mc{E}_\mc{P}, \mb{z} \mid \mb{X}, \theta^\mathrm{new}) \approx \prod_{e_{ij} \in \mc{E}_\mc{P}} \hat{y}_{ij} \prod_{e_{ij} \in \mc{E}_\mc{U}} \bigl( z_{ij} \hat{y}_{ij} + (1-z_{ij}) (1-\hat{y}_{ij}) \bigr)
\label{eq:theory-2}
\end{equation}
where $\hat{y}_{ij} = f_{\theta^\mathrm{new}}(i, j)$, and $z_{ij} \in \{ 1, 0 \}$ represents the connectivity between nodes $i$ and $j$.

Using Equations~(\ref{equation:mrf-2}) and (\ref{eq:theory-2}), we derive Theorem~\ref{theorem:1} that shows the similarity between the iterative minimization of $\mc{L}_E$ in \method and the iterative maximization of $Q$ in the EM.
\begin{theorem}
Given the assumption in Equation (\ref{eq:theory-2}), the likelihood $Q(\theta^\mathrm{new} \mid \theta)$ of the EM algorithm reduces to the negative of the loss $\mc{L}_E$ of \method with the expected graph $\mc{\bar{G}}$:
\begin{small}
\begin{multline}
	Q(\theta^\mathrm{new} \mid \theta) \approx \sum_{e_{ij} \in \mc{E}_\mc{P}} \log \hat{y}_{ij} + \\
	\sum_{e_{ij} \in \mc{E}_\mc{U}} \bigl( \mb{A}^\mc{\bar{G}}_{ij} \log \hat{y}_{ij} + (1-\mb{A}^\mc{\bar{G}}_{ij}) \log (1-\hat{y}_{ij}) \bigr)
\end{multline}
\end{small}
where $\hat{y}_{ij}$ is the estimated linking probability between nodes $i$ and $j$ by $f_{\theta^\mathrm{new}}$, and $\mb{A}^\mc{\bar{G}}$ is the corresponding adjacency matrix of $\mc{\bar{G}}$ (see Appendix for proof).
\label{theorem:1}
\end{theorem}

\subsubsection{Complexity of PULL.}
PULL is scalable to large graphs due to its linear scalability with the number of nodes and edges.
Let $n_o$ and $n_i$ be the number of outer and inner iterations in Algorithm~\ref{algorithm:proposed-method}, respectively.
For simplicity, we assume \method has $m$ layers where the number $d$ of features for each node is the same in all layers.

\begin{theorem}
The time complexity of \method (Algorithm~\ref{algorithm:proposed-method}) is
\begin{gather*}
	O \left( n_o d( (n_i m + r n_o n_i m)|\mc{E}_\mc{P}| + ( n_i m d + M ) |\mc{V}| )  \right),
\end{gather*}
which is linear to the numbers $|\mc{V}|$ and $|\mc{E}_\mc{P}|$ of nodes and edges in $\mc{G}_\mc{P}$, respectively (see Appendix for proof).
\label{theorem:time-complexity}
\end{theorem}

%% file: 040experiment.tex
\section{Experiments}
\label{sec:exp}

\begin{table*}[t]
\centering
\caption{
	The link prediction accuracy of \method and baselines in terms of AUROC and AUPRC.
	Bold numbers denote the best performance, and underlined ones represent the second-best accuracy.
	PULL outperforms all the baselines in most of the cases.
}

\scalebox{0.62}{
\begin{tabular}{@{}l|cccccccccc@{}}
\toprule
\multicolumn{11}{c}{Missing ratio $r_m$ = 0.1} \\ \midrule
\multicolumn{1}{l|}{\multirow{2}{*}{\textbf{Model}}} & \multicolumn{2}{c}{\textbf{PubMed}} & \multicolumn{2}{c}{\textbf{Cora-full}} & \multicolumn{2}{c}{\textbf{Chameleon}} & \multicolumn{2}{c}{\textbf{Crocodile}} & \multicolumn{2}{c}{\textbf{Facebook}} \\
 & AUROC & AUPRC & AUROC & AUPRC & AUROC & AUPRC & AUROC & AUPRC & AUROC & AUPRC \\ \midrule
GCN+CE & {\ul{96.45 $\pm$ 0.23}} & {\ul{96.58 $\pm$ 0.21}} & 95.77 $\pm$ 0.65 & 95.77 $\pm$ 0.74 & 96.77 $\pm$ 0.35 & 96.67 $\pm$ 0.40 & 96.91 $\pm$ 0.46 & 97.22 $\pm$ 0.45 & {\ul{97.06 $\pm$ 0.18}} & 97.33 $\pm$ 0.19 \\
GAT+CE & 90.99 $\pm$ 0.40 & 89.64 $\pm$ 0.49 & 94.27 $\pm$ 0.38 & 93.74 $\pm$ 0.43 & 91.55 $\pm$ 1.82 & 91.39 $\pm$ 1.72 & 90.65 $\pm$ 1.83 & 91.67 $\pm$ 1.40 & 92.43 $\pm$ 0.62 & 92.04 $\pm$ 0.77 \\
SAGE+CE & 87.22 $\pm$ 1.14 & 88.34 $\pm$ 0.99 & 94.35 $\pm$ 0.54 & 94.77 $\pm$ 0.60 & 96.30 $\pm$ 0.48 & 95.87 $\pm$ 0.63 & 96.00 $\pm$ 0.61 & 96.55 $\pm$ 0.55 & 95.17 $\pm$ 0.52 & 95.34 $\pm$ 0.54 \\
GAE & 96.35 $\pm$ 0.17 & 96.46 $\pm$ 0.15 & 95.51 $\pm$ 0.31 & 95.52 $\pm$ 0.32 & 96.88 $\pm$ 0.48 & 96.80 $\pm$ 0.54 & 96.67 $\pm$ 0.70 & 96.78 $\pm$ 1.17 & 97.00 $\pm$ 0.17 & 97.27 $\pm$ 0.13 \\
VGAE & 94.61 $\pm$ 1.01 & 94.74 $\pm$ 1.00 & 92.37 $\pm$ 3.89 & 92.40 $\pm$ 3.68 & 96.32 $\pm$ 0.27 & 96.20 $\pm$ 0.26 & 95.29 $\pm$ 0.40 & 95.45 $\pm$ 0.82 & 96.29 $\pm$ 0.27 & 96.49 $\pm$ 0.28 \\
ARGA & 93.67 $\pm$ 0.71 & 93.35 $\pm$ 0.73 & 91.39 $\pm$ 1.02 & 90.72 $\pm$ 1.15 & 94.76 $\pm$ 0.51 & 94.37 $\pm$ 0.71 & 96.03 $\pm$ 0.38 & 95.65 $\pm$ 0.35 & 92.03 $\pm$ 0.59 & 92.19 $\pm$ 0.48 \\
ARGVA & 93.56 $\pm$ 1.21 & 93.80 $\pm$ 1.11 & 89.88 $\pm$ 3.13 & 89.59 $\pm$ 2.88 & 94.26 $\pm$ 0.74 & 94.32 $\pm$ 0.70 & 95.04 $\pm$ 0.18 & 94.32 $\pm$ 0.59 & 92.35 $\pm$ 2.58 & 92.76 $\pm$ 2.36 \\
VGNAE & 95.99 $\pm$ 0.63 & 95.99 $\pm$ 0.55 & 95.42 $\pm$ 1.23 & 95.14 $\pm$ 1.34 & {\ul{97.46 $\pm$ 0.43}} & {\ul{97.17 $\pm$ 0.53}} & 96.34 $\pm$ 0.76 & 95.29 $\pm$ 1.87 & 95.79 $\pm$ 0.52 & 95.89 $\pm$ 0.54 \\
Bagging-PU & 94.55 $\pm$ 0.39 & 94.88 $\pm$ 0.37 & 92.74 $\pm$ 0.62 & 94.20 $\pm$ 0.77 & 97.27 $\pm$ 0.77 & 97.14 $\pm$ 0.89 & {\ul{97.47 $\pm$ 0.44}} & {\ul{97.75 $\pm$ 0.39}} & 97.02 $\pm$ 0.15 & {\ul{97.38 $\pm$ 0.14}} \\
NESS & 95.15 $\pm$ 0.24 & 95.01 $\pm$ 0.23 & \textbf{96.29 $\pm$ 0.20} & {\ul{96.16 $\pm$ 0.27}} & 96.86 $\pm$ 0.12 & 96.83 $\pm$ 0.09 & 95.88 $\pm$ 0.46 & 96.74 $\pm$ 0.34 & 95.75 $\pm$ 0.21 & 96.17 $\pm$ 0.14 \\ \midrule
PULL (proposed) & \textbf{96.59 $\pm$ 0.19} & \textbf{96.83 $\pm$ 0.18} & {\ul{96.06 $\pm$ 0.34}} & \textbf{96.25 $\pm$ 0.35} & \textbf{97.87 $\pm$ 0.33} & \textbf{97.83 $\pm$ 0.33} & \textbf{98.31 $\pm$ 0.20} & \textbf{98.36 $\pm$ 0.22} & \textbf{97.41 $\pm$ 0.11} & \textbf{97.67 $\pm$ 0.09} \\ \bottomrule
\end{tabular}
}

\label{table:performance}
\end{table*}

\begin{table*}[t]
\centering
\caption{
	The performance improvement of baselines with the integration of \method.
	The best performance is indicated in bold.
	Note that \method enhances the performance of baseline models.
}

\scalebox{0.62}{
\begin{tabular}{@{}lcccccccccc@{}}
\toprule
\multicolumn{11}{c}{Missing ratio $r_m$ = 0.1} \\ \midrule
\multicolumn{1}{l|}{\multirow{2}{*}{\textbf{Model}}} & \multicolumn{2}{c}{\textbf{PubMed}} & \multicolumn{2}{c}{\textbf{Cora-full}} & \multicolumn{2}{c}{\textbf{Chameleon}} & \multicolumn{2}{c}{\textbf{Crocodile}} & \multicolumn{2}{c}{\textbf{Facebook}} \\
\multicolumn{1}{l|}{} & AUROC & AUPRC & AUROC & AUPRC & AUROC & AUPRC & AUROC & AUPRC & AUROC & AUPRC \\ \midrule
\multicolumn{1}{l|}{GAE} & 96.35 $\pm$ 0.17 & 96.46 $\pm$ 0.15 & 95.51 $\pm$ 0.31 & 95.52 $\pm$ 0.32 & 96.88 $\pm$ 0.48 & 96.80 $\pm$ 0.54 & 96.67 $\pm$ 0.70 & 96.78 $\pm$ 1.17 & 97.00 $\pm$ 0.17 & 97.27 $\pm$ 0.13 \\
\multicolumn{1}{l|}{GAE+PULL} & \textbf{96.64 $\pm$ 0.22} & \textbf{96.86 $\pm$ 0.21} & \textbf{96.00 $\pm$ 0.48} & \textbf{96.12 $\pm$ 0.58} & \textbf{98.04 $\pm$ 0.18} & \textbf{98.05 $\pm$ 0.15} & \textbf{98.22 $\pm$ 0.18} & \textbf{98.31 $\pm$ 0.17} & \textbf{97.41 $\pm$ 0.14} & \textbf{97.67 $\pm$ 0.11} \\ \midrule
\multicolumn{1}{l|}{VGAE} & 94.61 $\pm$ 1.01 & 94.74 $\pm$ 1.00 & 92.37 $\pm$ 3.89 & 92.40 $\pm$ 3.68 & 96.32 $\pm$ 0.27 & 96.20 $\pm$ 0.26 & 95.29 $\pm$ 0.40 & 95.45 $\pm$ 0.82 & 96.29 $\pm$ 0.27 & 96.49 $\pm$ 0.28 \\
\multicolumn{1}{l|}{VGAE+PULL} & \textbf{95.81 $\pm$ 0.51} & \textbf{95.92 $\pm$ 0.50} & \textbf{93.75 $\pm$ 3.17} & \textbf{93.85 $\pm$ 3.01} & \textbf{97.24 $\pm$ 0.47} & \textbf{97.29 $\pm$ 0.49} & \textbf{97.17 $\pm$ 0.73} & \textbf{97.33 $\pm$ 0.64} & \textbf{96.56 $\pm$ 0.25} & \textbf{96.72 $\pm$ 0.26} \\ \midrule
\multicolumn{1}{l|}{VGNAE} & 95.99 $\pm$ 0.63 & 95.99 $\pm$ 0.55 & 95.42 $\pm$ 1.23 & 95.14 $\pm$ 1.34 & 97.46 $\pm$ 0.43 & 97.17 $\pm$ 0.53 & 96.34 $\pm$ 0.76 & 95.29 $\pm$ 1.87 & 95.79 $\pm$ 0.52 & 95.89 $\pm$ 0.54 \\
\multicolumn{1}{l|}{VGNAE+PULL} & \textbf{96.22 $\pm$ 0.37} & \textbf{96.23 $\pm$ 0.37} & \textbf{96.02 $\pm$ 0.64} & \textbf{95.75 $\pm$ 0.70} & \textbf{97.75 $\pm$ 0.36} & \textbf{97.46 $\pm$ 0.43} & \textbf{97.23 $\pm$ 0.73} & \textbf{96.91 $\pm$ 0.79} & \textbf{95.83 $\pm$ 0.46} & \textbf{95.91 $\pm$ 0.44} \\ \bottomrule
\end{tabular}
}

\label{table:baseline-improvement}
\end{table*}

We conduct experiments to answer the following questions:

\begin{enumerate*}
	\item[\textbf{Q1.}] \textbf{Link prediction performance.}
		How accurate is \method compared to the baselines for predicting links in edge-incomplete graphs?
	\item[\textbf{Q2.}] \textbf{Applying \method to other baselines.}
		Does applying \method to other methods improve the accuracy?
	\item[\textbf{Q3.}] \textbf{Effect of iterative learning.}
		How does the accuracy of \method change over iterations?
	\item[\textbf{Q4.}] \textbf{Effect of additional loss.}
		How does the additional loss term $\mc{L}_C$ of \method contribute to the performance?
	\item[\textbf{Q5.}] \textbf{Scalability.}
		How does the running time of \method change as the graph size grows?
\end{enumerate*}

\subsection{Experimental Settings}
\label{ssec:exp-settings}

\paragraph{\textbf{Datasets.}}
We use seven real-world datasets from various domains which are summarized in Table~4 (in Appendix).
PubMed and Cora-full are citation networks where nodes correspond to scientific publications and edges signify citation relationships between the papers.
Each node has binary bag-of-words features.
Chameleon and Crocodile are Wikipedia networks, with nodes representing web pages and edges representing hyperlinks between them.
Node features include keywords or informative nouns extracted from the pages.
Facebook is a social network where nodes represent users, and edges indicate mutual likes.
Node features represent user-specific information such as age and gender.

\paragraph{\textbf{Baselines.}} We compare \method with previous approaches for link prediction.
GCN+CE, GAT+CE, and SAGE+CE use GCN~\shortcite{KipfW17}, GAT~\shortcite{Petar17}, and GraphSAGE~\shortcite{HamiltonYL17} for computing linking probabilities, respectively.
Cross entropy (CE) loss is used and $|\mc{E}_\mc{P}|$ non-edges are sampled randomly from $\mc{E}_\mc{U}$ every epoch to balance the ratio between edge and non-edge examples.
GAE and VGAE~\shortcite{KipfW16a} are autoencoder-based frameworks for link prediction.
ARGA and ARGVA~\shortcite{PanHLJYZ18} respectively improve the performance of GAE and VGAE by introducing adversarial training strategy.
VGNAE~\shortcite{AhnK21} utilizes L2-normalization to obtain better node embeddings for isolated nodes.
Bagging-PU~\shortcite{gan2022positive} classifies node pairs into observed and unobserved, and approximates the linking probabilities using the ratio of observed links.
NESS~\shortcite{abs-2303-08958} is the current state-of-the-art link prediction method.
It utilizes multiple static subgraphs during training, and aggregates the learned node embeddings to obtain the linking probabilities at test time.
More implementation details are described in the Appendix section.

\paragraph{\textbf{Evaluation and Settings.}}
We evaluate the performance of \method and the baselines in classifying edges and non-edges correctly.
We use AUROC (AUC) and AUPRC (AP) scores as the evaluation metric~\cite{KipfW17}.
The models are trained using graphs that lack some edges, while preserving all node attributes.
The validation and test sets consist of the missing edges and an equal number of randomly sampled non-edges.
We vary the ratio $r_m$ of test missing edges in \{0.1, 0.2\}.
The ratio of valid missing edges are set to 0.1 through the experiments.
The validation set is employed for early stopping with patience 20, and the number of maximum epochs is set to 2,000.
The epochs are not halted until 500 since the accuracy oscillates in the earlier epochs.
For \method, we set the number of inner loops as 200, and the maximum number of iterations as 10.
The iterations stop if the current validation AUROC is smaller than that of the previous iteration.
We use Adam optimizer with a learning rate of 0.01, and set the numbers of layers and hidden dimensions as 2 and 16, respectively, following the original GCN paper~\shortcite{KipfW17} for fair comparison between the methods.
For the hyperparameters of the baselines, we use the default settings described in their papers.
We repeat the experiments ten times with different random seeds and present the results in terms of both average and standard deviation.

\begin{figure*}[t]
	\centering
	\includegraphics[width=0.9\textwidth]{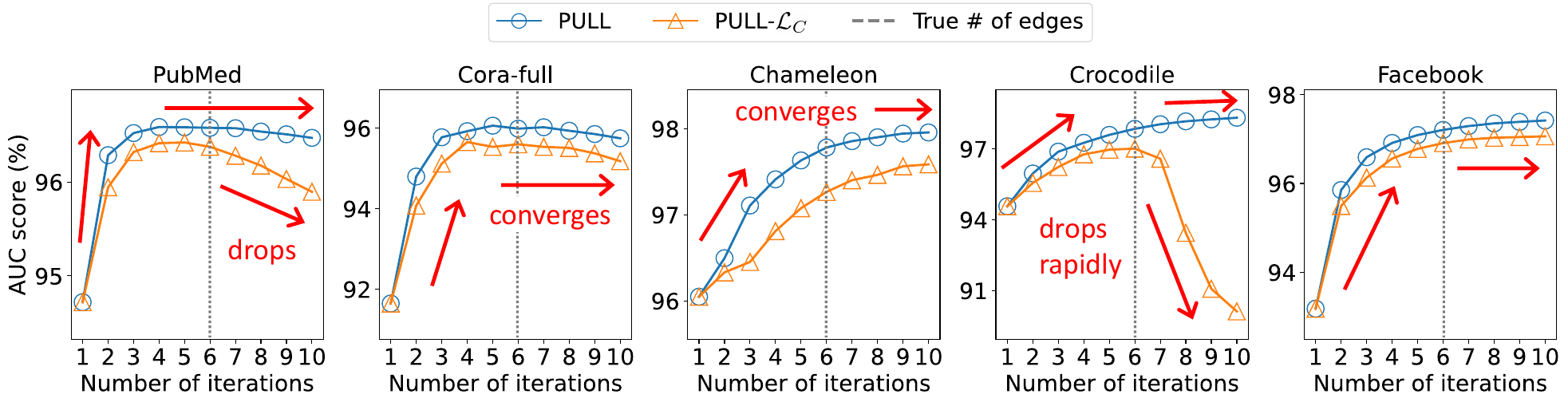}
	\caption{
		AUC score of \method and \method-$\mc{L}_C$ through the iterations.
		\method-$\mc{L}_C$ represents \method without $\mc{L}_C$.
		The dashed gray lines denote the ground-truth numbers of edges.
		The accuracy of \method increases in early iterations, and converges or slightly increases as the number $K$ of sampled edges exceeds the ground-truth one.
		This shows that \method improves the quality of the expected graph with each iteration.
		Moreover, \method consistently shows superior performance than \method-$\mc{L}_C$.
		In PubMed and Crocodile, the accuracy of \method-$\mc{L}_C$ drops rapidly after exceeding the dashed gray lines.
		This demonstrates that $\mc{L}_C$ protects \method from performance degradation when the expected graph structure has more edges than the actual graph.
	}
\label{figure:iteration}
\end{figure*}

\subsection{Link Prediction Performance (Q1)}
\label{ssec:exp-performance}

We compare the link prediction performance of \method with the baselines in Table~\ref{table:performance} \blue{while setting the ratio $r_m$ of test missing edges as 0.1.
We also report the performance with $r_m = 0.2$ in Appendix.}
Note that \method achieves the highest AUROC and AUPRC scores among the methods in most of the cases.
This highlights the significance of considering the uncertainty in the given edge-incomplete graph to enhance the prediction performance.
It is also noteworthy that GCN+CE model, which propagates information through the edge-incomplete graph using GCN, shows consistently lower performance than \method.
This shows that the propagation of \method with the expected graph structure effectively prevents $f$ from overtrusting the given graph structure, whereas the propagation of GCN+CE with the given graph leads to overfitting.

\subsection{Applying \method to Other Methods (Q2)}
\label{ssec:apply-method}

\method can be applied to other GCN-based methods including GAE, VGAE, and VGNAE that use GCN-based propagation scheme.
To show that \method improves the performance of existing models, we conduct experiments by applying \method to them \blue{while setting the ratio $r_m$ of test missing edges as 0.1.
We also report the experimental results with $r_m = 0.2$ in Appendix.}
It is not easy for \method to be directly applied to other baselines such as GAT, GraphSAGE, ARGA, and ARGVA.
This is because they use different propagation schemes instead of GCN, posing a challenge for \method in propagating information through the expected graph structure during training.
Table~\ref{table:baseline-improvement} summarizes the results.
Note that \method improves the accuracy of the baselines in most of the cases, demonstrating its effectiveness across a range of different models.


\begin{figure}[t]
\centering
	\includegraphics[width=0.33\textwidth]{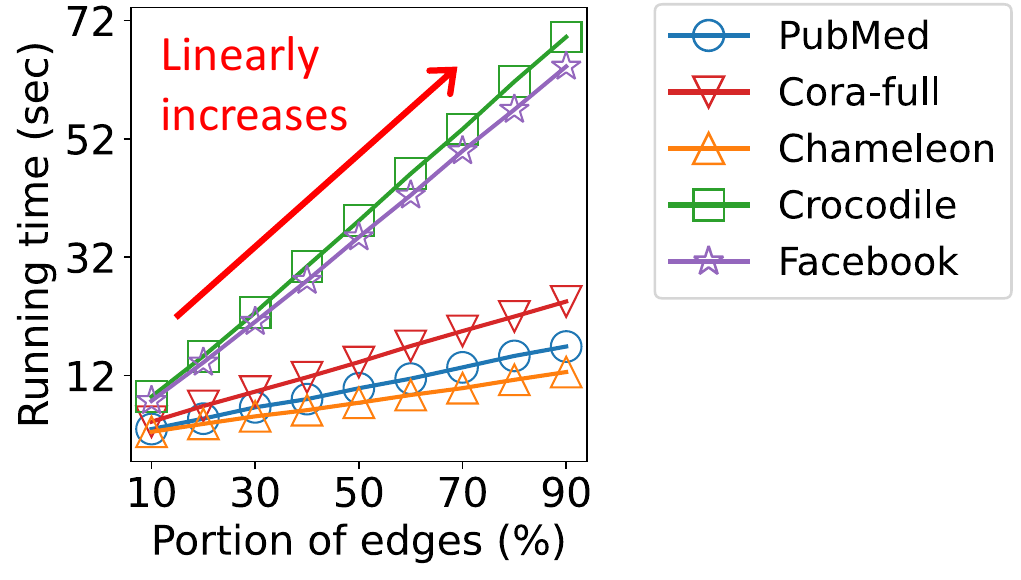}
	\caption{
		The running time of \method on sampled subgraphs.
		The time increases linearly with the number of edges.
	}
\label{figure:scalability}
\end{figure}

\subsection{Effect of Iterative Learning (Q3)}
\label{ssec:exp-iterative}

For each iteration, \method computes the expected graph $\mc{\bar{G}}$ utilizing the trained link predictor $f$ from the previous iteration.
Then \method retrains $f$ with $\mc{\bar{G}}$ to prevent the link predictor from blindly trusting the given graph.
We study how the accuracy of \method evolves as the iteration proceeds in Figure~\ref{figure:iteration}.
\method increases the number $K$ of selected edges for the approximation of $\mc{\bar{G}}$ as the iteration progresses.
The dashed gray lines indicate the points at which $K$ becomes equal to the ground-truth number of edges for each dataset.

The AUC score of \method in Figure~\ref{figure:iteration} increases through the iterations, reaching its best performance when the number $K$ of selected edges closely matches the ground-truth one.
This shows that \method enhances the quality of the expected graph as the iterations progress, and eventually makes accurate predictions of the true graph structure.
In Crocodile and Facebook, the prediction accuracy increases even with larger number of edges than the ground-truth one.
This observation indicates that the ground-truth graph structures of Crocodile and Facebook inherently contain missing links.

\subsection{Effect of Additional Loss (Q4)}
\label{ssec:exp-loss}

We study the effect of the additional loss term $\mc{L}_C$ of \method on the link prediction performance in Figure~\ref{figure:iteration}.
\method-$\mc{L}_C$ represents \method trained by minimizing only $\mc{L}_E'$.
Note that \method-$\mc{L}_C$ consistently shows lower accuracy than \method, showing the importance of $\mc{L}_C$.
In PubMed and Crocodile, the AUC scores of \method-$\mc{L}_C$ drop rapidly after the fifth iteration, where the number $K$ of selected edges exceeds the ground-truth one.
This indicates that $\mc{L}_C$ effectively safeguards \method against performance degradation when the expected graph contains more edges than the actual one.

\subsection{Scalability (Q5)}
\label{ssec:exp-scalability}

We study the running time of \method on subgraphs with different sizes to show its scalability to large graphs in Figure~\ref{figure:scalability}.
We randomly sample edges from the original graphs with various portions $r_p \in \{0.1, ..., 0.9 \}$.
Thus, each induced subgraph has $r_p |\mc{E}|$ edges where $\mc{E}$ is the set of edges of the original graph.
Figure~\ref{figure:scalability} shows that the runtime of \method exhibits a linear increase with the number of edges, showing its scalability to large graphs.
This is because \method effectively approximates the expected graph $\bar{\mc{G}}$ with $|\mc{V}|^2$ weighted edges by a graph $\bar{\mc{G}}'$ with $(1+0.05t)|\mc{E}_\mc{P}|$ edges where $\mc{V}$ and $\mc{E}_\mc{P}$ are sets of nodes and observed edges, respectively, and $t$ is the number of iterations. 

%% file: 050conclusion.tex
\section{Conclusion}
\label{sec:conclusion}

We propose \method, an accurate method for link prediction in edge-incomplete graphs.
\method addresses the limitation of previous approaches, which is their heavy reliance on the observed graph, by iteratively predicting the true graph structure.
\method proposes latent variables for the unconnected edges in a graph, and propagates information through the expected graph structure.
\method then uses the expected linking probabilities of unconnected edges as their pseudo labels for training a link predictor.
Extensive experiments on real-world datasets show that \method shows superior performance than the baselines.
%
%

%% file: 060supp.tex
\setcounter{table}{2}
\setcounter{equation}{9}


\begin{abstract}
We provide additional information including the problem definition, symbols, lemma, proofs, detailed experimental settings, and further experimental results.
\end{abstract}

\section{Problem Definition}
\label{sec:appendix-problem}

\begin{problem}[Link Prediction in Edge-incomplete Graphs]
We have an edge-incomplete graph $\mc{G}_\mc{P}=(\mc{V}, \mc{E}_\mc{P})$, along with a feature matrix $\mb{X} \in \mathbb{R}^{|\mc{V}| \times d}$ where $\mc{V}$ and $\mc{E}_\mc{P}$ are the sets of nodes and observed edges, respectively, and $d$ is the number of features for each node.
The remaining unconnected node pairs are denoted as a set $\mc{E}_\mc{U}$.
The objective of link prediction in edge-incomplete graphs is to train a link predictor $f$ that accurately identifies the connected node pairs within $\mc{E}_\mc{U}$.
\label{problem:link-prediction}
\end{problem}

\section{Lemma}
\label{sec:appendix-lemma}

\begin{lemma}
	We are given a graph $\mc{G}_\mc{P}$ and its corresponding line graph $L(\mc{G}_\mc{P})=(\mc{V}_L, \mc{E}_L)$ where $\mc{V}_L$ and $\mc{E}_L$ are sets of nodes and edges in $L(\mc{G}_\mc{P})$, respectively.
	We are also given node potentials $\phi_{ij}(z_{ij} \mid \theta)$ of nodes $ij$ in graph $L(\mc{G}_\mc{P})$.
	Then the equation $\sum_{\mb{z}}\prod_{ij \in \mc{V}_L} \phi_{ij}(z_{ij} \mid \theta) = 1$ holds for $\sum_{z_{ij}} \phi_{ij}(z_{ij} \mid \theta) = 1$.
\label{appendix-lemma1}
\end{lemma}

\begin{proof}
Let $N = |\mc{V}_L|$, and $\mc{E}$ be the set of all observed edges and unconnected edges in $\mc{G}_\mc{P}$.
Then the sum of $\prod_{ij \in \mc{V}_L} \phi_{ij}(z_{ij}|\theta)$ for all possible ${\mb{z}}$ is computed as follows:
\begin{equation*}
\small
\begin{split}
		& \sum_{\mb{z}} \prod_{ij \in \mc{V}_L} \phi_{ij}(z_{ij} \mid \theta)
		= \sum_{\mb{z}} \prod_{e_{ij} \in \mc{E}} \phi_{ij}(z_{ij} \mid \theta)
		= \sum_{\mb{z} \setminus \{{z_{11}}\}} \prod_{e_{ij} \in \mc{E} \setminus \{e_{11}\}} \phi_{ij}(z_{ij} \mid \theta) \sum_{z_{11}} \phi_{11}(z_{11} \mid \theta) \\
		&= \sum_{\mb{z} \setminus \{z_{11},z_{12}\}} \prod_{e_{ij} \in \mc{E}\setminus \{e_{11},e_{12}\}} \phi_{ij}(z_{ij} \mid \theta) \sum_{z_{12}} \phi_{12}(z_{12}\mid \theta)
		= \dots = \sum_{z_{NN}} \phi_{NN}(z_{NN}\mid \theta) = 1
	\end{split}
	\label{eq:lemma1-proof-1}
\end{equation*}
which ends the proof.
Similarly, the equation $\sum_{\mb{z} \mid z_{ij}=1} \prod_{e_{kl} \in \mc{E}_\mc{U} \setminus \{e_{ij}\}} \phi_{kl}(z_{kl} \mid \theta) = 1$ holds for $\sum_{z_{ij}} \phi_{ij}(z_{ij} \mid \theta) = 1$.
\end{proof}

\section{Proofs for Theorems}
\subsection{Proof of Theorem 1}
\label{ssec:appendix-proof1}

\begin{proof}
	Using Equations~(2) and (8), the expected log likelihood $Q(\theta^\mathrm{new} \mid \theta)$ in Equation (7) is expressed as follows:
	\begin{equation}
	\small
	\begin{split}
		Q(\theta^\mathrm{new} \mid \theta)
		= \sum_{\mb{z}} p( \mb{z} \mid \mb{X}, \mc{E}_\mc{P}, \theta) \log p(\mc{E}_\mc{P}, \mb{z} \mid \mb{X}, \theta^\mathrm{new})
		&\approx \sum_\mb{z} p(\mb{z} \mid \mb{X}, \mc{E}_\mc{P}, \theta) \Bigl( \sum_{e_{ij} \in \mc{E}_\mc{P}} \log \hat{y}_{ij} + \sum_{e_{ij} \in \mc{E}_\mc{U}} \log \bigl(\hat{y}_{ij}(z_{ij})\bigr) \Bigr) \\
		&= \sum_{e_{ij} \in \mc{E}_\mc{P}} \log \hat{y}_{ij} + \sum_\mb{z} \prod_{e_{kl} \in \mc{E}_\mc{U}} \phi_{kl}(z_{kl} \mid \theta) \sum_{e_{ij} \in \mc{E}_\mc{U}} \log \bigl(\hat{y}_{ij}(z_{ij})\bigr)
	\end{split}
	\label{eq:theorem1-proof-1}
	\end{equation}
	where $\hat{y}_{ij}(z_{ij}) = z_{ij}\hat{y}_{ij} + (1-z_{ij})(1-\hat{y}_{ij})$.

	The last term in Equation~\eqref{eq:theorem1-proof-1} is expressed as follows:
	\begin{equation}
	\small
	\begin{split}
		\sum_\mb{z} \prod_{e_{kl} \in \mc{E}_\mc{U}} \phi_{kl}(z_{kl} \mid \theta) \sum_{e_{ij} \in \mc{E}_\mc{U}} \log \bigl(\hat{y}_{ij}(z_{ij})\bigr)
		&= \sum_\mb{z} \prod_{e_{kl} \in \mc{E}_\mc{U} \setminus \{e_{ij}\}} \phi_{kl}(z_{kl} \mid \theta) \sum_{e_{ij} \in \mc{E}_\mc{U}} \phi_{ij}(z_{ij} \mid \theta) \log \bigl(\hat{y}_{ij}(z_{ij})\bigr) \\
		&= \sum_{e_{ij} \in \mc{E}_\mc{U}} \bigl( \phi_{ij}(z_{ij}=1 \mid \theta) \log \hat{y}_{ij} + \phi_{ij}(z_{ij}=0 \mid \theta) \log (1-\hat{y}_{ij}) \bigr) \\
		&= \sum_{e_{ij} \in \mc{E}_\mc{U}} \bigl( \mb{A}^\mc{\bar{G}}_{ij} \log \hat{y}_{ij} + (1-\mb{A}^\mc{\bar{G}}_{ij}) \log (1-\hat{y}_{ij}) \bigr)
	\end{split}
	\label{eq:theorem1-proof-2}
	\end{equation}	
	where the second equality uses the fact that
	$\sum_{\mb{z} \setminus \{z_{ij}\}} \prod_{e_{kl} \in \mc{E}_\mc{U} \setminus \{e_{ij}\}}$ $\phi_{kl}(z_{kl} \mid \theta) = 1$
	(from Lemma~\ref{appendix-lemma1}), and the third equality uses Equation~(4).

	Using the result of Equation~(\ref{eq:theorem1-proof-2}), the expected log likelihood $Q(\theta^\mathrm{new} \mid \theta)$ reduces to the negative of the loss function $\mc{L}_E$ of \method:
	\small
	\begin{equation}
		Q(\theta^\mathrm{new} \mid \theta) \approx \sum_{e_{ij} \in \mc{E}_\mc{P}} \log \hat{y}_{ij}
		+ \sum_{e_{ij} \in \mc{E}_\mc{U}} \bigl( \mb{A}^\mc{\bar{G}}_{ij} \log \hat{y}_{ij} + (1-\mb{A}^\mc{\bar{G}}_{ij}) \log (1-\hat{y}_{ij}) \bigr)
	\label{eq:theorem1-proof-3}
	\end{equation}
	which ends the proof.
	Note that Equation~(\ref{eq:theorem1-proof-3}) uses $\mc{\bar{G}}$ which is approximated to $\mc{\bar{G}}'$ in \method.
\end{proof}

\subsection{Proof of Theorem~2}
\label{ssec:appendix-proof2}


\begin{proof}
	Let $n_o$ be the number of iterations in Algorithm~1, and $n_i$ be the number of gradient-descent updates for obtaining the model parameter $\theta^\mathrm{new}$ in line 7 of the algorithm.
	Each iteration of \method consists of two steps: 1) generating the expected graph structure $\bar{\mc{G}}$, and 2) training the link predictor $f$ using the approximated one of $\bar{\mc{G}}$.
	The time complexity of generating the expected graph structure is $O(d(M|\mc{V}|))$ since we compute the linking probabilities for node pairs where at least one node of each pair belongs to the set of nodes with top-$M$ largest degree.
	The time complexity for training $f$ in the $k$-th iteration is $O \left( n_i m d ((1+r k)|\mc{E}_\mc{P}| + d|\mc{V}|) \right)$ where $r$ is the increase factor of edges for approximating the expected graph structure $\bar{\mc{G}}$. 
	Note that the complexity for training $f$ is upper-bounded by $O \left( n_i m d ((1+r n_o)|\mc{E}_\mc{P}| + d|\mc{V}|) \right)$ since $1 + r k \leq 1 + r n_o$.
	As a result, the time complexity of \method is computed as
	\begin{equation*}
	\small
		O \left( n_o d( (n_i m + r n_o n_i m)|\mc{E}_\mc{P}| + ( n_i m d + M ) |\mc{V}| )  \right),
	\end{equation*}
	which ends the proof.
	
\end{proof}

\section{Symbols and Datasets}
\label{sec:appendix-data}

\begin{table*}[h]
\centering
\caption{Symbols and dataset summary.}

\begin{subtable}[t]{0.45\textwidth}
    \centering
    \caption{Symbols.}
    \scalebox{0.8}{
    \begin{tabular}{c|l}
        \toprule
        \textbf{Symbol} & \textbf{Description} \\ \midrule
        $\mc{G}_\mc{P} = (\mc{V}, \mc{E}_\mc{P})$ & Edge-incomplete graph with sets $\mc{V}$ of nodes and \\ 
        & $\mc{E}_\mc{P}$ of observed edges \\ 
        $\mc{E}_\mc{U}$ & Set of unconnected node pairs (unconnected edges) \\
        $e_{ij}$ & Edge between nodes $i$ and $j$ \\
        $L(\mc{G})$ & Corresponding line graph of $\mc{G}$ \\ \midrule
        $\mb{A}^{\mc{G}}$ & Adjacency matrix of $\mc{G}=(\mc{V}, \mc{E})$ \\ 
        & where $\mb{A}^{\mc{G}}_{ij} = 1$ if $e_{ij} \in \mc{E}$ \\ 
        $\mb{X}$ & Feature matrix for every node in $\mc{G}_\mc{P}$ \\ \midrule
        $f_\theta(\cdot, \cdot)$ & Link predictor parameterized by $\theta$ \\
        $\mc{L}(\cdot)$ & Objective function that \method aims to minimize \\
        $\bar{\mc{G}}$ & Expected graph structure \\
        $\bar{\mc{G}}'$ & Approximated version of $\bar{\mc{G}}$ \\ \bottomrule
    \end{tabular}
    }
    \label{subtable:symbols}
\end{subtable}
\hfill
\begin{subtable}[t]{0.45\textwidth}
    \centering
    \caption{Summary of datasets.}
    \begin{threeparttable}
	\scalebox{0.85}{
    \begin{tabular}{@{}l|rrrr@{}}
        \toprule
        \textbf{Datasets} & \textbf{Nodes} & \textbf{Edges} & \textbf{Features} & \textbf{Description} \\ \midrule
        PubMed\tnote{1} & 19,717 & 88,648 & 500 & Citation \\
        Cora-full\tnote{2} & 19,793 & 126,842 & 8,710 & Citation \\
        Chameleon\tnote{3} & 2,277 & 36,101 & 2,325 & Wikipedia \\
        Crocodile\tnote{3} & 11,631 & 191,506 & 500 & Wikipedia \\
        Facebook\tnote{4} & 22,470 & 342,004 & 128 & Social \\ \midrule
        Physics\tnote{5} & 34,493 & 495,924 & 8,415 & Citation \\
        ogbn-arxiv\tnote{6} & 169,343 & 1,166,243 & 128 & Citation \\ \bottomrule
    \end{tabular}
    }
    \begin{tablenotes} \scriptsize
        \item[1]\url{https://github.com/kimiyoung/planetoid}
        \item[2]\url{https://www.cs.cit.tum.de/daml/g2g/}
        \item[3]\url{https://snap.stanford.edu/data/wikipedia-article-networks.html}
        \item[4]\url{https://github.com/benedekrozemberczki/MUSAE}
        \item[5]\url{https://github.com/shchur/gnn-benchmark/raw/master/data/npz/}
        \item[6]\url{https://ogb.stanford.edu/docs/nodeprop/#ogbn-arxiv}
    \end{tablenotes}
    \end{threeparttable}
    \label{subtable:datasets}
\end{subtable}

\label{table:combined}
\end{table*}

\section{Detailed Settings of Experiments}
\label{sec:appendix-detail-setting}


The statistics of datasets are summarized in Table~\ref{subtable:datasets}.

\subsection{Implementation}
\label{sec:appendix-implementation}

We provide detailed settings of implementation for \method and the baselines.
All the experiments are conducted under a single GPU machine with GTX 1080 Ti.

\textbf{GCN+CE.} We use the GCN code implemented with torch-geometric package~\citep{Fey/Lenssen/2019}.
For each epoch, the model randomly samples $|\mc{E}_\mc{P}|$ negative samples (unconnected node pairs), and minimizes the cross entropy (CE) loss.

\textbf{GAT+CE.} We use the GAT code provided by torch-geometric package.
For each epoch, GAT+CE randomly samples $|\mc{E}_\mc{P}|$ negative samples, and minimizes the cross entropy loss.
We set the multi-head attention number as 8 with the mean aggregation strategy, and the dropout ratio as 0.6 following the original paper~\citep{Petar17}.

\textbf{SAGE+CE.} We use the GraphSAGE code implemented with torch-geometric package.
For each epoch, the model randomly samples $|\mc{E}_\mc{P}|$ negative samples, and minimizes the cross entropy loss.
We use the mean aggregation following the original paper~\citep{HamiltonYL17}.

\textbf{GAE \& VGAE.} We use the GAE and VGAE codes implemented with torch-geometric package.
We use GCN-based encoder and decoder for both GAE and VGAE following the original paper~\citep{KipfW16a}.
The number of layers and units for decoders are set to 2 and 16, respectively.

\textbf{ARGA \& ARGVA.} We use the ARGA and ARGVA codes implemented with torch-geometric package.
We use the same hyperparameter settings for the adversarial training as presented in the original paper~\citep{PanHLJYZ18}.

\textbf{VGNAE.} We use the VGNAE code implemented by the authors~\citep{AhnK21}.
The scaling constant $s$ is set to 1.8 following the original paper.

\textbf{Bagging-PU.} We reimplement Bagging-PU since there is no public implementation of authors.
We use GCN instead of SDNE~\citep{wang2016structural} for the node embedding model since SDNE is an unsupervised representation-based method, which limits the performance.
We use the mean aggregation following the original paper~\citep{gan2022positive}, and set the bagging size as 3.

\textbf{NESS.} We use the NESS code implemented by the authors~\citep{abs-2303-08958}.
We use GCN-based embedding model.
For the other hyperparameters, we use the default settings described in their paper.

\textbf{\method.} We use torch-geometric package to implement the weighted propagation of GCN.
The number of inner epochs is set to 200, while that of outer iteration is set to 10.
We increase the number $K$ of edges in the approximated version of expected graph $\bar{\mc{G}}$ in proportion to that of observed edges through the iterations: $K \leftarrow K + r|\mc{E}_\mc{P}|$ where $r$ is the increasing ratio.
We set $r=0.05$ in our experiments.
For the number $M$ of candidate nodes for generating the candidate edges, we set $M=100$.

\section{Further Experiments}
\label{sec:further-experiments}

\begin{table*}[t]
\centering
\caption{
    The link prediction accuracy of \method and the baselines on larger datasets.
    Bold numbers denote the best performance, and underlined ones denote the second-best performance.
    Note that \method demonstrates the highest accuracy across various settings, showing its efficacy in larger graphs.
}

\begin{subtable}[t]{0.48\textwidth}
    \centering
    \caption{Missing ratio $r_m$ = 0.1}
    \scalebox{0.75}{
    \begin{tabular}{@{}lcccc@{}}
        \toprule
        \multicolumn{1}{l|}{\multirow{2}{*}{\textbf{Model}}} & \multicolumn{2}{c}{\textbf{Physics}} & \multicolumn{2}{c}{\textbf{ogbn-arxiv}} \\
        \multicolumn{1}{l|}{} & AUROC & AUPRC & AUROC & AUPRC \\ \midrule
        \multicolumn{1}{l|}{GCN+CE} & 96.90 $\pm$ 0.19 & \ul{96.65 $\pm$ 0.23} & 80.58 $\pm$ 0.13 & 85.11 $\pm$ 0.10 \\
        \multicolumn{1}{l|}{GAT+CE} & 93.58 $\pm$ 0.46 & 92.23 $\pm$ 0.52 & 82.31 $\pm$ 0.22 & 79.46 $\pm$ 0.43 \\
        \multicolumn{1}{l|}{SAGE+CE} & 95.40 $\pm$ 0.47 & 94.95 $\pm$ 0.49 & 83.07 $\pm$ 1.60 & 81.01 $\pm$ 1.07 \\
        \multicolumn{1}{l|}{GAE} & 96.81 $\pm$ 0.13 & 96.56 $\pm$ 0.14 & 80.62 $\pm$ 0.14 & 85.20 $\pm$ 0.11 \\
        \multicolumn{1}{l|}{VGAE} & 95.00 $\pm$ 0.82 & 94.51 $\pm$ 0.89 & 80.29 $\pm$ 0.32 & 83.83 $\pm$ 0.27 \\
        \multicolumn{1}{l|}{ARGA} & 91.72 $\pm$ 0.61 & 90.57 $\pm$ 0.51 & 83.09 $\pm$ 1.18 & 86.13 $\pm$ 0.77 \\
        \multicolumn{1}{l|}{ARGVA} & 92.56 $\pm$ 1.38 & 91.84 $\pm$ 1.47 & 82.77 $\pm$ 1.71 & 85.74 $\pm$ 1.77 \\
        \multicolumn{1}{l|}{VGNAE} & 94.68 $\pm$ 0.69 & 93.87 $\pm$ 0.74 & 77.37 $\pm$ 0.10 & 81.43 $\pm$ 0.07 \\
        \multicolumn{1}{l|}{Bagging-PU} & 95.86 $\pm$ 0.20 & 96.00 $\pm$ 0.27 & 81.25 $\pm$ 0.24 & 85.47 $\pm$ 0.10 \\
        \multicolumn{1}{l|}{NESS} & \ul{96.96 $\pm$ 0.08} & 96.56 $\pm$ 0.10 & \ul{85.91 $\pm$ 0.08} & \ul{87.90 $\pm$ 0.12} \\ \midrule
        \multicolumn{1}{l|}{PULL (ours)} & \textbf{97.27 $\pm$ 0.07} & \textbf{97.12 $\pm$ 0.10} & \textbf{86.46 $\pm$ 0.04} & \textbf{89.08 $\pm$ 0.44} \\ \bottomrule
    \end{tabular}
    }
\end{subtable}
\hfill
\begin{subtable}[t]{0.48\textwidth}
    \centering
    \caption{Missing ratio $r_m$ = 0.2}
    \scalebox{0.75}{
    \begin{tabular}{@{}lcccc@{}}
        \toprule
        \multicolumn{1}{l|}{\multirow{2}{*}{\textbf{Model}}} & \multicolumn{2}{c}{\textbf{Physics}} & \multicolumn{2}{c}{\textbf{ogbn-arxiv}} \\
        \multicolumn{1}{l|}{} & AUROC & AUPRC & AUROC & AUPRC \\ \midrule
        \multicolumn{1}{l|}{GCN+CE} & 96.60 $\pm$ 0.09 & 96.32 $\pm$ 0.11 & 80.36 $\pm$ 0.20 & 84.99 $\pm$ 0.14 \\
        \multicolumn{1}{l|}{GAT+CE} & 93.55 $\pm$ 0.41 & 92.19 $\pm$ 0.49 & 82.49 $\pm$ 0.07 & 79.64 $\pm$ 0.22 \\
        \multicolumn{1}{l|}{SAGE+CE} & 95.13 $\pm$ 0.35 & 94.67 $\pm$ 0.42 & 83.31 $\pm$ 2.03 & 81.30 $\pm$ 1.42 \\
        \multicolumn{1}{l|}{GAE} & 96.57 $\pm$ 0.20 & 96.31 $\pm$ 0.25 & 80.43 $\pm$ 0.44 & 85.06 $\pm$ 0.28 \\
        \multicolumn{1}{l|}{VGAE} & 94.30 $\pm$ 0.59 & 93.73 $\pm$ 0.60 & 79.38 $\pm$ 0.30 & 83.29 $\pm$ 0.26 \\
        \multicolumn{1}{l|}{ARGA} & 91.75 $\pm$ 0.39 & 90.49 $\pm$ 0.52 & 82.91 $\pm$ 0.79 & 86.04 $\pm$ 0.47 \\
        \multicolumn{1}{l|}{ARGVA} & 92.65 $\pm$ 1.19 & 91.94 $\pm$ 1.26 & 81.57 $\pm$ 1.55 & 84.43 $\pm$ 0.90 \\
        \multicolumn{1}{l|}{VGNAE} & 94.48 $\pm$ 0.71 & 93.69 $\pm$ 0.69 & 76.58 $\pm$ 0.15 & 81.01 $\pm$ 0.12 \\
        \multicolumn{1}{l|}{Bagging-PU} & 95.59 $\pm$ 0.12 & 95.77 $\pm$ 0.14 & 80.84 $\pm$ 0.35 & 85.26 $\pm$ 0.22 \\
        \multicolumn{1}{l|}{NESS} & \ul{96.96 $\pm$ 0.08} & \ul{96.56 $\pm$ 0.10} & \textbf{85.81 $\pm$ 0.04} & \ul{87.81 $\pm$ 0.03} \\ \midrule
        \multicolumn{1}{l|}{PULL (ours)} & \textbf{97.01 $\pm$ 0.05} & \textbf{96.89 $\pm$ 0.07} & \ul{84.95 $\pm$ 0.66} & \textbf{88.80 $\pm$ 0.74} \\ \bottomrule
    \end{tabular}
    }
\end{subtable}

\label{table:large-graph}
\end{table*}

\begin{table*}[t]
\centering
\caption{
	The link prediction accuracy of \method and its variant \method-WS.
	\method-WS approximates $\bar{\mc{G}}$ by performing weighted random sampling of edges based on the linking probabilities.
	Bold numbers denote the best performance.
	\method outperforms \method-WS in every case.
}

\begin{subtable}[t]{0.48\textwidth}
    \centering
    \caption{Missing ratio $r_m$ = 0.1}
    \scalebox{0.75}{
    \begin{tabular}{@{}lllll@{}}
        \toprule
        \multicolumn{1}{l|}{\multirow{2}{*}{\textbf{Dataset}}} & \multicolumn{2}{c|}{\textbf{PULL-WS}} & \multicolumn{2}{c}{\textbf{PULL (proposed)}} \\
        \multicolumn{1}{l|}{} & \multicolumn{1}{c}{AUROC} & \multicolumn{1}{c|}{AUPRC} & \multicolumn{1}{c}{AUROC} & \multicolumn{1}{c}{AUPRC} \\ \midrule
        \multicolumn{1}{l|}{PubMed} & 96.54 $\pm$ 0.18 & \multicolumn{1}{l|}{96.80 $\pm$ 0.14} & \textbf{96.59 $\pm$ 0.19} & \textbf{96.83 $\pm$ 0.18} \\
        \multicolumn{1}{l|}{Cora-full} & 95.94 $\pm$ 0.31 & \multicolumn{1}{l|}{96.12 $\pm$ 0.32} & \textbf{96.06 $\pm$ 0.34} & \textbf{96.25 $\pm$ 0.35} \\
        \multicolumn{1}{l|}{Chameleon} & 97.69 $\pm$ 0.28 & \multicolumn{1}{l|}{97.68 $\pm$ 0.28} & \textbf{97.87 $\pm$ 0.33} & \textbf{97.83 $\pm$ 0.33} \\
        \multicolumn{1}{l|}{Crocodile} & 97.38 $\pm$ 0.31 & \multicolumn{1}{l|}{97.66 $\pm$ 0.26} & \textbf{98.31 $\pm$ 0.20} & \textbf{98.36 $\pm$ 0.22} \\
        \multicolumn{1}{l|}{Facebook} & 97.05 $\pm$ 0.15 & \multicolumn{1}{l|}{97.30 $\pm$ 0.14} & \textbf{97.41 $\pm$ 0.11} & \textbf{97.67 $\pm$ 0.09} \\ \bottomrule
    \end{tabular}
    }
\end{subtable}
\hfill
\begin{subtable}[t]{0.48\textwidth}
    \centering
    \caption{Missing ratio $r_m$ = 0.2}
    \scalebox{0.75}{
    \begin{tabular}{@{}lllll@{}}
        \toprule
        \multicolumn{1}{l|}{\multirow{2}{*}{\textbf{Dataset}}} & \multicolumn{2}{c|}{\textbf{PULL-WS}} & \multicolumn{2}{c}{\textbf{PULL (proposed)}} \\
        \multicolumn{1}{l|}{} & \multicolumn{1}{c}{AUROC} & \multicolumn{1}{c|}{AUPRC} & \multicolumn{1}{c}{AUROC} & \multicolumn{1}{c}{AUPRC} \\ \midrule
        \multicolumn{1}{l|}{PubMed} & 96.24 $\pm$ 0.14 & \multicolumn{1}{l|}{96.43 $\pm$ 0.14} & \textbf{96.28 $\pm$ 0.13} & \textbf{96.47 $\pm$ 0.17} \\
        \multicolumn{1}{l|}{Cora-full} & 95.31 $\pm$ 0.35 & \multicolumn{1}{l|}{95.62 $\pm$ 0.33} & \textbf{95.39 $\pm$ 0.32} & \textbf{95.65 $\pm$ 0.31} \\
        \multicolumn{1}{l|}{Chameleon} & 97.66 $\pm$ 0.18 & \multicolumn{1}{l|}{97.65 $\pm$ 0.18} & \textbf{97.89 $\pm$ 0.14} & \textbf{97.87 $\pm$ 0.16} \\
        \multicolumn{1}{l|}{Crocodile} & 97.28 $\pm$ 0.22 & \multicolumn{1}{l|}{97.57 $\pm$ 0.21} & \textbf{98.19 $\pm$ 0.13} & \textbf{98.29 $\pm$ 0.16} \\
        \multicolumn{1}{l|}{Facebook} & 96.95 $\pm$ 0.10 & \multicolumn{1}{l|}{97.20 $\pm$ 0.09} & \textbf{97.30 $\pm$ 0.07} & \textbf{97.59 $\pm$ 0.06} \\ \bottomrule
    \end{tabular}
    }
\end{subtable}

\label{table:weighted-random-sampling}
\end{table*}

\subsection{Link Prediction in Larger Network}
\label{sec:large-network}

We additionally perform link prediction on larger graph datasets compared to those used in Table~1.
The ogbn-arxiv dataset is a citation network consisting of 169,343 nodes and 1,166,243 edges, where each node represents an arXiv paper and an edge indicates that one paper cites another one.
Each node has 128-dimensional feature vector, which is derived by averaging the embeddings of the words in its title and abstract.
Physics is a co-authorship graph based on the Microsoft Academic Graph from the KDD Cup 2016 challenge 3.
Physics contains 34,493 nodes and 495,924 edges where each node represents an author, and they are connected if they co-authored a paper.
We summarize the statistics of the larger networks in Table~\ref{table:datasets}. 
For \method, we set the maximum number of iterations as 20.
For the baselines, we set the maximum number of epochs as 4,000.
This is because a larger data size requires a greater number of epochs to train the link predictor.
\method incorporates the early stopping with a patience of one, for more stable learning.
For other cases, we used the same settings as in the Experiments section.

Table~\ref{table:large-graph} presents the link prediction performance of \method and the baselines in ogbn-arxiv and Physics.
Note that \method consistently shows superior performance than the baselines in terms of both AUROC and AUPRC.
This indicates that \method is also effective in handling larger graphs.

\subsection{Weighted Random Sampling for Constructing $\bar{\mc{G}}'$}
\label{sec:weighted-random-sampling}

\method keeps the top-$K$ edges with the highest linking probabilities to approximate $\bar{\mc{G}}$.
In this section, we compare \method with \method-WS (\method with Weighted Sampling) that constructs the approximated version $\bar{\mc{G}}'$ by performing weighted random sampling of edges from $\bar{\mc{G}}$ based on the linking probabilities.
As the weighted random sampling empowers \method to mitigate the excessive self-reinforcement in the link predictor, we additionally exclude the loss term $\mc{L}_C$, which serves the same purpose.
We conduct experiments five times with random seeds, while using the same experimental settings as in the Experiments section.

Table~\ref{table:weighted-random-sampling} shows that \method-WS presents marginal performance decrease compared to \method.
This indicates that keeping the top-$K$ edges having the highest linking probabilities with an additional loss term $\mc{L}_C$ shows better link prediction performance than performing weighted random sampling of edges without $\mc{L}_C$.
However, \method-WS is an efficient variant of \method that uses only a single loss term $\mc{L}'_E$ instead of the proposed loss $\mc{L} = \mc{L}_E' + \mc{L}_C$.

\blue{
\subsection{Performance with other ratio $r_m$ of missing edges}
\label{sec:other-missing-ratio}

\begin{table*}[t]
\centering
\caption{
	\blue{
	The link prediction accuracy of \method and the baselines in terms of AUROC and AUPRC, with the ratio $r_m$ of test missing edges set to $ r_m = 0.2 $. 
	Bold values indicate the best performance, while underlined values represent the second-best accuracy. 
	PULL achieves superior performance, outperforming all baselines in most cases.
	}
}

\scalebox{0.65}{
\begin{tabular}{@{}lcccccccccc@{}}
\toprule
\multicolumn{11}{c}{Missing ratio $r_m$ = 0.2} \\ \midrule
\multicolumn{1}{l|}{\multirow{2}{*}{\textbf{Model}}} & \multicolumn{2}{c}{\textbf{PubMed}} & \multicolumn{2}{c}{\textbf{Cora-full}} & \multicolumn{2}{c}{\textbf{Chameleon}} & \multicolumn{2}{c}{\textbf{Crocodile}} & \multicolumn{2}{c}{\textbf{Facebook}} \\
\multicolumn{1}{l|}{} & AUROC & AUPRC & AUROC & AUPRC & AUROC & AUPRC & AUROC & AUPRC & AUROC & AUPRC \\ \midrule
\multicolumn{1}{l|}{GCN+CE} & \ul{96.14 $\pm$ 0.19} & \ul{96.25 $\pm$ 0.21} & 94.92 $\pm$ 0.64 & 95.01 $\pm$ 0.75 & 96.85 $\pm$ 0.36 & 96.78 $\pm$ 0.45 & 97.06 $\pm$ 0.46 & 97.37 $\pm$ 0.40 & \ul{97.00 $\pm$ 0.23} & 97.26 $\pm$ 0.22 \\
\multicolumn{1}{l|}{GAT+CE} & 90.67 $\pm$ 0.37 & 89.32 $\pm$ 0.45 & 93.99 $\pm$ 0.35 & 93.47 $\pm$ 0.40 & 91.75 $\pm$ 1.65 & 91.28 $\pm$ 1.42 & 91.09 $\pm$ 1.50 & 91.80 $\pm$ 1.09 & 92.41 $\pm$ 0.48 & 92.19 $\pm$ 0.56 \\
\multicolumn{1}{l|}{SAGE+CE} & 85.90 $\pm$ 0.67 & 87.22 $\pm$ 0.93 & 93.71 $\pm$ 0.60 & 94.25 $\pm$ 0.66 & 96.11 $\pm$ 0.51 & 95.68 $\pm$ 0.63 & 95.92 $\pm$ 0.67 & 96.48 $\pm$ 0.62 & 94.96 $\pm$ 0.46 & 95.06 $\pm$ 0.55 \\
\multicolumn{1}{l|}{GAE} & 96.10 $\pm$ 0.15 & 96.22 $\pm$ 0.21 & 95.15 $\pm$ 0.39 & 95.24 $\pm$ 0.48 & 96.76 $\pm$ 0.42 & 96.60 $\pm$ 0.57 & 96.36 $\pm$ 0.65 & 96.74 $\pm$ 0.56 & 96.87 $\pm$ 0.38 & 97.12 $\pm$ 0.37 \\
\multicolumn{1}{l|}{VGAE} & 94.12 $\pm$ 1.13 & 94.17 $\pm$ 1.10 & 91.71 $\pm$ 3.94 & 91.73 $\pm$ 3.72 & 96.21 $\pm$ 0.22 & 96.01 $\pm$ 0.32 & 95.21 $\pm$ 0.45 & 95.40 $\pm$ 0.86 & 95.89 $\pm$ 0.54 & 96.11 $\pm$ 0.52 \\
\multicolumn{1}{l|}{ARGA} & 93.00 $\pm$ 0.58 & 92.43 $\pm$ 0.54 & 90.93 $\pm$ 0.62 & 90.40 $\pm$ 0.63 & 94.72 $\pm$ 0.34 & 94.37 $\pm$ 0.41 & 95.98 $\pm$ 0.47 & 95.63 $\pm$ 0.39 & 91.90 $\pm$ 0.51 & 91.98 $\pm$ 0.46 \\
\multicolumn{1}{l|}{ARGVA} & 93.19 $\pm$ 1.30 & 93.38 $\pm$ 1.17 & 87.56 $\pm$ 4.49 & 87.53 $\pm$ 4.21 & 94.07 $\pm$ 0.51 & 94.09 $\pm$ 0.40 & 94.85 $\pm$ 0.14 & 94.00 $\pm$ 0.15 & 92.68 $\pm$ 1.82 & 93.11 $\pm$ 1.69 \\
\multicolumn{1}{l|}{VGNAE} & 95.70 $\pm$ 0.39 & 95.62 $\pm$ 0.38 & \ul{95.40 $\pm$ 1.04} & 95.13 $\pm$ 1.06 & 97.45 $\pm$ 0.30 & \ul{97.13 $\pm$ 0.35} & 96.41 $\pm$ 0.77 & 95.91 $\pm$ 1.36 & 95.22 $\pm$ 0.88 & 95.33 $\pm$ 0.87 \\
\multicolumn{1}{l|}{Bagging-PU} & 94.02 $\pm$ 0.34 & 94.38 $\pm$ 0.41 & 92.56 $\pm$ 0.54 & 94.48 $\pm$ 0.67 & \ul{97.13 $\pm$ 0.47} & 97.08 $\pm$ 0.54 & \ul{97.48 $\pm$ 0.41} & \ul{97.79 $\pm$ 0.37} & 96.95 $\pm$ 0.21 & \ul{97.31 $\pm$ 0.21} \\
\multicolumn{1}{l|}{NESS} & 95.62 $\pm$ 0.04 & 95.45 $\pm$ 0.03 & \textbf{96.07 $\pm$ 0.24} & \textbf{96.09 $\pm$ 0.27} & 96.68 $\pm$ 0.29 & 96.64 $\pm$ 0.29 & 96.07 $\pm$ 0.11 & 96.78 $\pm$ 0.03 & 95.79 $\pm$ 0.06 & 96.20 $\pm$ 0.03 \\ \midrule
\multicolumn{1}{l|}{PULL (proposed)} & \textbf{96.28 $\pm$ 0.13} & \textbf{96.47 $\pm$ 0.17} & 95.39 $\pm$ 0.32 & \ul{95.65 $\pm$ 0.31} & \textbf{97.89 $\pm$ 0.14} & \textbf{97.87 $\pm$ 0.16} & \textbf{98.19 $\pm$ 0.13} & \textbf{98.29 $\pm$ 0.16} & \textbf{97.30 $\pm$ 0.07} & \textbf{97.59 $\pm$ 0.06} \\ \bottomrule
\end{tabular}
}

\label{table:performance-ratio}
\end{table*}

\begin{table*}[t]
\centering
\caption{
	\blue{
	The performance improvement of baseline models with the integration of \method, with the ratio $r_m$ of test missing edges set to $ r_m = 0.2 $.
	The best performance is highlighted in bold.
	\method consistently enhances the performance of the baseline models across all evaluations.
	}
}

\scalebox{0.65}{
\begin{tabular}{@{}lcccccccccc@{}}
\toprule
\multicolumn{11}{c}{Missing ratio $r_m$ = 0.2} \\ \midrule
\multicolumn{1}{l|}{\multirow{2}{*}{\textbf{Model}}} & \multicolumn{2}{c}{\textbf{PubMed}} & \multicolumn{2}{c}{\textbf{Cora-full}} & \multicolumn{2}{c}{\textbf{Chameleon}} & \multicolumn{2}{c}{\textbf{Crocodile}} & \multicolumn{2}{c}{\textbf{Facebook}} \\
\multicolumn{1}{l|}{} & AUROC & AUPRC & AUROC & AUPRC & AUROC & AUPRC & AUROC & AUPRC & AUROC & AUPRC \\ \midrule
\multicolumn{1}{l|}{GAE} & 96.10 $\pm$ 0.15 & 96.22 $\pm$ 0.21 & 95.15 $\pm$ 0.39 & 95.24 $\pm$ 0.48 & 96.76 $\pm$ 0.42 & 96.60 $\pm$ 0.57 & 96.36 $\pm$ 0.65 & 96.74 $\pm$ 0.56 & 96.87 $\pm$ 0.38 & 97.12 $\pm$ 0.37 \\
\multicolumn{1}{l|}{GAE+PULL} & \textbf{96.23 $\pm$ 0.10} & \textbf{96.47 $\pm$ 0.12} & \textbf{95.44 $\pm$ 0.41} & \textbf{95.69 $\pm$ 0.51} & \textbf{98.00 $\pm$ 0.15} & \textbf{98.03 $\pm$ 0.15} & \textbf{98.18 $\pm$ 0.19} & \textbf{98.31 $\pm$ 0.17} & \textbf{97.26 $\pm$ 0.12} & \textbf{97.53 $\pm$ 0.12} \\ \midrule
\multicolumn{1}{l|}{VGAE} & 94.12 $\pm$ 1.13 & 94.17 $\pm$ 1.10 & 91.71 $\pm$ 3.94 & 91.73 $\pm$ 3.72 & 96.21 $\pm$ 0.22 & 96.01 $\pm$ 0.32 & 95.21 $\pm$ 0.45 & 95.40 $\pm$ 0.86 & 95.89 $\pm$ 0.54 & 96.11 $\pm$ 0.52 \\
\multicolumn{1}{l|}{VGAE+PULL} & \textbf{95.30 $\pm$ 0.65} & \textbf{95.35 $\pm$ 0.66} & \textbf{93.19 $\pm$ 3.46} & \textbf{93.37 $\pm$ 3.30} & \textbf{96.97 $\pm$ 0.56} & \textbf{96.97 $\pm$ 0.63} & \textbf{97.24 $\pm$ 0.67} & \textbf{97.44 $\pm$ 0.54} & \textbf{96.51 $\pm$ 0.23} & \textbf{96.67 $\pm$ 0.23} \\ \midrule
\multicolumn{1}{l|}{VGNAE} & 95.70 $\pm$ 0.39 & 95.62 $\pm$ 0.38 & 95.40 $\pm$ 1.04 & 95.13 $\pm$ 1.06 & 97.45 $\pm$ 0.30 & 97.13 $\pm$ 0.35 & 96.41 $\pm$ 0.77 & 95.91 $\pm$ 1.36 & 95.22 $\pm$ 0.88 & 95.33 $\pm$ 0.87 \\
\multicolumn{1}{l|}{VGNAE+PULL} & \textbf{95.84 $\pm$ 0.31} & \textbf{95.74 $\pm$ 0.26} & \textbf{95.65 $\pm$ 0.70} & \textbf{95.42 $\pm$ 0.73} & \textbf{97.70 $\pm$ 0.31} & \textbf{97.36 $\pm$ 0.36} & \textbf{96.67 $\pm$ 1.32} & \textbf{96.13 $\pm$ 2.18} & \textbf{95.72 $\pm$ 0.44} & \textbf{95.78 $\pm$ 0.41} \\ \bottomrule
\end{tabular}
}

\label{table:baseline-improvement-ratio}
\end{table*}

In the experiment section, we demonstrated the superior performance of \method with a test missing edge ratio $r_m = 0.1$.
To further validate the robustness of \method across varying levels of missing lines, we conducted additional experiments with a larger ratio, $r_m = 0.2$.

The results in Table~\ref{table:performance-ratio} highlight that \method consistently achieves the best link prediction accuracy compared to other baselines.
Furthermore, as shown in Table~\ref{table:baseline-improvement-ratio}, the integration of \method into existing baseline methods significantly enhances their performance, underscoring the adaptability and robustness of \method across different levels of missing data.
}